\documentclass[twoside,11pt]{article}

\usepackage[T1]{fontenc}
\usepackage{custom_no_natbib}
\usepackage{microtype}
\usepackagewithoutnatbib{jmlr2e}
\usepackage{custom_citations}
\usepackage{custom_commands}

\bibliography{references}

\usepackage{nicefrac}
\usepackage{caption}
\usepackage{lastpage}

\usepackage[dvipsnames]{xcolor}
\definecolor{urlblue}{rgb}{0,0.08,0.45}
\renewcommand\UrlFont{\color{urlblue}\rmfamily}

\jmlrheading{22}{2021}{1-\pageref{LastPage}}{11/20}{5/21}{20-1260}{James T. Wilson, Viacheslav Borovitskiy, Alexander Terenin, Peter Mostowsky, and Marc Peter Deisenroth}
\ShortHeadings{Pathwise Conditioning of Gaussian Processes}
{Wilson, Borovitskiy, Terenin, Mostowsky, and Deisenroth}
\firstpageno{1}

\begin{document}

\title{Pathwise Conditioning of Gaussian Processes}

\author{\name James T. Wilson\textsuperscript{\ensuremath{*}}
\email j.wilson17@imperial.ac.uk \\
\addr Imperial College London
\AND
\name Viacheslav Borovitskiy\textsuperscript{\ensuremath{*}}
\email viacheslav.borovitskiy@gmail.com \\
\addr St. Petersburg State University and\newline St. Petersburg Department of Steklov Mathematical Institute of Russian Academy of Sciences
\AND
\name Alexander Terenin\textsuperscript{\ensuremath{*}}
\email a.terenin17@imperial.ac.uk \\
\addr Imperial College London
\AND
\name Peter Mostowsky\textsuperscript{\ensuremath{*}}
\email pmostowsky@gmail.com \\
\addr St. Petersburg State University
\AND
\name Marc Peter Deisenroth
\email m.deisenroth@ucl.ac.uk \\
\addr Centre for Artificial Intelligence, University College London
}

\editor{Kilian Weinberger}
\maketitle

\begin{table}[b!]
\vspace*{-1.5ex}
\footnoterule
\footnotesize\textsuperscript{\ensuremath{*}}Equal contribution.
\end{table}

\begin{abstract}As Gaussian processes are used to answer increasingly complex questions, analytic solutions become scarcer and scarcer. 
Monte Carlo methods act as a convenient bridge for connecting intractable mathematical expressions with actionable estimates via sampling.
Conventional approaches for simulating Gaussian process posteriors view samples as draws from marginal distributions of process values at finite sets of input locations. 
This distribution-centric characterization leads to generative strategies that scale cubically in the size of the desired random vector.
These methods are prohibitively expensive in cases where we would, ideally, like to draw high-dimensional vectors or even continuous sample paths.
In this work, we investigate a different line of reasoning: rather than focusing on distributions, we articulate Gaussian conditionals at the level of random variables.
We show how this \emph{pathwise} interpretation of conditioning gives rise to a general family of approximations that lend themselves to efficiently sampling Gaussian process posteriors. 
Starting from first principles, we derive these methods and analyze the approximation errors they introduce. 
We, then, ground these results by exploring the practical implications of pathwise conditioning in various applied settings, such as global optimization and reinforcement learning.
\end{abstract}


\section{Introduction}
\label{sec:introduction}

In machine learning, the narrative of Gaussian processes (GPs) is dominated by talk of distributions \cite{rasmussen06}. This view is often helpful and convenient: a Gaussian process is a random function; however, seeing as we may trivially marginalize out arbitrary subsets of this function, we can simply focus on its behavior at a finite number of input locations. When dealing with regression and classification problems, this reduction simplifies discourse and expedites implementation by allowing us to work with joint distributions at training and test locations instead of random functions.

Model-based learning and prediction generally service broader goals. For example, when making decisions in the face of uncertainty, models enable us to simulate the consequences of our actions. 
Decision-making, then, amounts to optimizing the expectation of a simulated quantity of interest, such as a cost or a reward. 
Be it for purposes of safety or for balancing trade-offs between long-term and short-term goals, it is crucial that these simulations faithfully portray both knowledge and uncertainty. 
Gaussian processes are known to make accurate, well-calibrated predictions and, therefore, stand as the model-of-choice in fields such as Bayesian optimization \cite{shahriari2015taking}, uncertainty quantification \cite{bect2012sequential}, and model-based reinforcement learning \cite{Deisenroth2015}.

Unfortunately, marginal distributions and simulations do not always go hand in hand. 
When the quantity of interest is a function of a process value $f(\v{x}_{*})$ at an individual input location $\v{x}_{*}$, its expectation can sometimes be obtained analytically. 
Conversely, when this quantity is a function of process values $\v{f}_{*} = f(\m{X}_{*})$ at multiple locations $\m{X}_{*}$, its expectation is generally intractable. 
Rather than solving these integrals directly in terms of marginal distributions $p(\v{f}_{*})$, we therefore estimate them by averaging over many simulations of $\v{f}_{*}$. Drawing $\v{f}_{*}$ from $p(\v{f}_{*})$ takes $\c{O}(*^{3})$ time, where $* = \vert \m{X}_{*} \vert$ is the number of input locations. 
Hence, distribution-based approaches to sampling $\v{f}_{*}$ quickly become untenable as this number increases. 
In these cases, we may be better off thinking about GPs from a perspective that naturally lends itself to sampling

In the early 1970s, one such view surfaced in the then nascent field of geostatistics \cite{journel1978mining, chiles12}. Instead of emphasizing the statistical properties of Gaussian random variables, "conditioning by Kriging" encourages us to think in terms of the variables themselves. We study the broader implications of this paradigm shift to develop a general framework for conditioning Gaussian processes at the level of random functions. Formulating conditioning in terms of sample paths, rather than distributions, allows us to separate out the effect of the prior from that of the data. By leveraging this property, we can use \emph{pathwise conditioning} to efficiently approximate function draws from GP posteriors. As we will see, working with sample paths enables us to simulate process values $\v{f}_{*}$ in $\c{O}(*)$ time and brings with it a host of additional benefits.

The structure of the remaining text is as follows. Section~\ref{sec:mvn_conditioning} and Section~\ref{sec:gp_conditioning} introduce pathwise conditioning of Gaussian random vectors and processes, respectively. Section~\ref{sec:priors} surveys strategies for approximating function draws from GP priors, while Section~\ref{sec:updates} discusses methods for mapping from prior to posterior random variables. Section~\ref{sec:error} studies the behavior of errors introduced by different approximation techniques, and Section~\ref{sec:applications} complements this theory with a bit of empiricism by exploring a number of examples. 
Section~\ref{sec:conclusion} concludes.

\paragraph{Notation}

By way of example, we denote matrices as $\m{A}$ and vectors as $\v{a}$. 
We write $\v{x} = \v{a} \oplus \v{b}$ for the direct sum (i.e. concatenation) of vectors $\v{a}$ and $\v{b}$. 
Throughout, we use $\vert \. \vert$ to denote the cardinality of sets and dimensionality of vectors. 
When dealing with covariance matrices $\m{\Sigma} = \Cov(\v{x}, \v{x})$, we use subscripts to identify corresponding blocks. 
For example, $\m{\Sigma}_{\v{a},\v{b}} = \Cov(\v{a}, \v{b})$. 
As shorthand, we denote the evaluation of a function $f : \c{X} \to \mathbb{R}$ at a finite set of locations $\m{X}_{*} \subset \c{X} $ by the vector $\v{f}_{*}$. 
Putting these together, when dealing with random variables $\v{f}_{*} = f(\m{X}_{*})$ and $\v{f}_{n} = f(\m{X}_{n})$, we write $\m{K}_{*, n} = \Cov(\v{f}_{*}, \v{f}_{n})$.

\section{Conditioning Gaussian distributions and random vectors}
\label{sec:mvn_conditioning}
A random vector $\v{x} = (x_{1}, \ldots, x_{n}) \in \R^n$ is said to be Gaussian if there exists a matrix $\m{L}$ and vector $\v{\mu}$ such that 
\[
\label{eqn:location_scale}
\v{x} &\eqd \v{\mu} + \m{L}\v{\zeta}
&
\v{\zeta} &\~\c{N}(\v{0},\m{I}),
\]
where $\c{N}(\v{0},\m{I})$ is the standard (multivariate) normal distribution, the probability density function of which is given below. Each such distribution is uniquely identified by its first two moments: its mean $\v{\mu} = \E(\v{x})$ and its covariance $\m{\Sigma} = \E[(\v{x} - \v{\mu})(\v{x} - \v{\mu})^{\top}]$.
Assuming it exists, the corresponding density function is defined as
\[
p(\v{x})
=
    \c{N}(\v{x} \given \v{\mu}, \m{\Sigma})
= 
    \frac{1}{\sqrt{\left\vert 2\pi \m{\Sigma} \right\vert}}
    \exp\left(-\frac{1}{2}(\v{x} - \v{\mu})^\top \m{\Sigma}^{-1} (\v{x} - \v{\mu})\right).
\label{eqn:gaussian_density}
\]
The representation of $\v{x}$ given by~\eqref{eqn:location_scale} is commonly referred to as its \emph{location-scale} form and stands as the most widely used method for generating Gaussian random vectors. Since $\v{\zeta}$ has identity covariance, any matrix square root of $\m{\Sigma}$, such as its Cholesky factor $\m{L}$ with $\m{\Sigma} = \m{L}\m{L}^{\top}$, may be used to draw $\v{x}$ as prescribed by \eqref{eqn:location_scale}.

Here, we focus on multivariate cases $n > 1$ and investigate different ways of reasoning about random variables $\v{a} \given \v{b} = \v{\beta}$ for non-trivial partitions $\v{x} = \v{a} \oplus \v{b}$.

\subsection{Distributional conditioning}
\label{sec:mvn_conditioning_distributional}

The quintessential approach to deriving the distribution of $\v{a}$ subject to the condition $\v{b} = \v{\beta}$ begins by employing the usual set of matrix identities to factor $p(\v{b})$ from $p(\v{a}, \v{b})$. Applying Bayes' rule, $p(\v{b})$ then cancels out and $p(\v{a} \given \v{b} = \v{\beta})$ is identified as the remaining term---namely, the Gaussian distribution $\c{N}(\v{\mu}_{\v{a} \given \v{\beta}}, \m{\Sigma}_{\v{a}, \v{a} \given \v{\beta}})$ with moments
\[
\label{eqn:mvn_cond_moments}
\v{\mu}_{\v{a} \given \v{\beta}} 
&= 
    \v{\mu}_{\v{a}} 
    + \m{\Sigma}_{\v{a}, \v{b}}\m{\Sigma}_{\v{b}, \v{b}}^{-1}(\v{\beta} - \v{\mu}_{\v{b}})
&
\m{\Sigma}_{\v{a},\v{a} \given \v{\beta}}
&= 
    \m{\Sigma}_{\v{a},\v{a}} - 
    \m{\Sigma}_{\v{a}, \v{b}}^{\vphantom{{-1}}}
    \m{\Sigma}_{\v{b}, \v{b}}^{-1}
    \m{\Sigma}_{\v{b}, \v{a}}^{\vphantom{{-1}}}
.
\]
Having obtained this conditional distribution, we can now generate $\v{a} \given \v{b} = \v{\beta}$ by computing a matrix square root of $\m{\Sigma}_{\v{a},\v{a} \given \v{\beta}}$ and constructing a location-scale transform \eqref{eqn:location_scale}.

Due to their emphasis of conditional distributions, we refer to methods that represent or generate a random variable $\v{a} \given \v{b} = \v{\beta}$ by way of $p(\v{a} \given \v{b} = \v{\beta})$ as being \emph{distributional} in kind.
This approach to conditioning is not only standard, but particularly natural when quantities of interest may be derived analytically from $p(\v{a} \given \v{b} = \v{\beta})$. 
Many quantities, such as expectations of nonlinear functions, cannot be deduced analytically from $p(\v{a} \given \v{b} = \v{\beta})$ alone, however. In these case, we must instead work with realizations of $\v{a} \given \v{b} = \v{\beta}$. Since the cost of obtaining a matrix square root of $\m{\Sigma}_{\v{a},\v{a} \given \v{\beta}}$ scales cubically in $\vert \v{a} \vert$, distributional approaches to evaluating these quantities struggle to accommodate high-dimensional random vectors. To address this issue, we now consider Gaussian conditioning in another light.

\subsection{Pathwise conditioning}
\label{sec:mvn_conditioning_pathwise}
\begin{figure}
    \centering
    \begin{minipage}{0.45\textwidth}
        \centering
        \includegraphics[width=\textwidth]{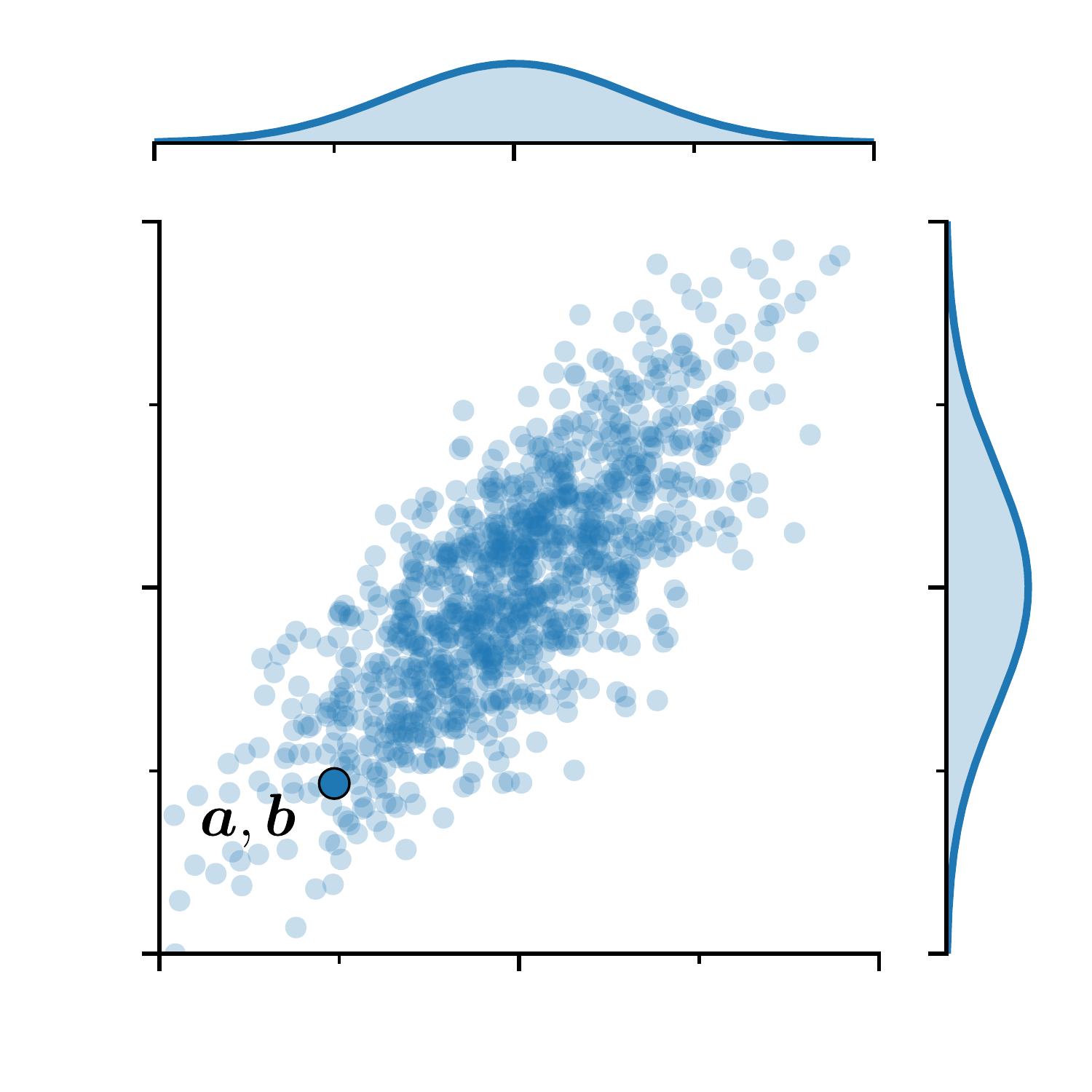}
    \end{minipage}\hfill
    \begin{minipage}{0.45\textwidth}
        \centering
        \includegraphics[width=\textwidth]{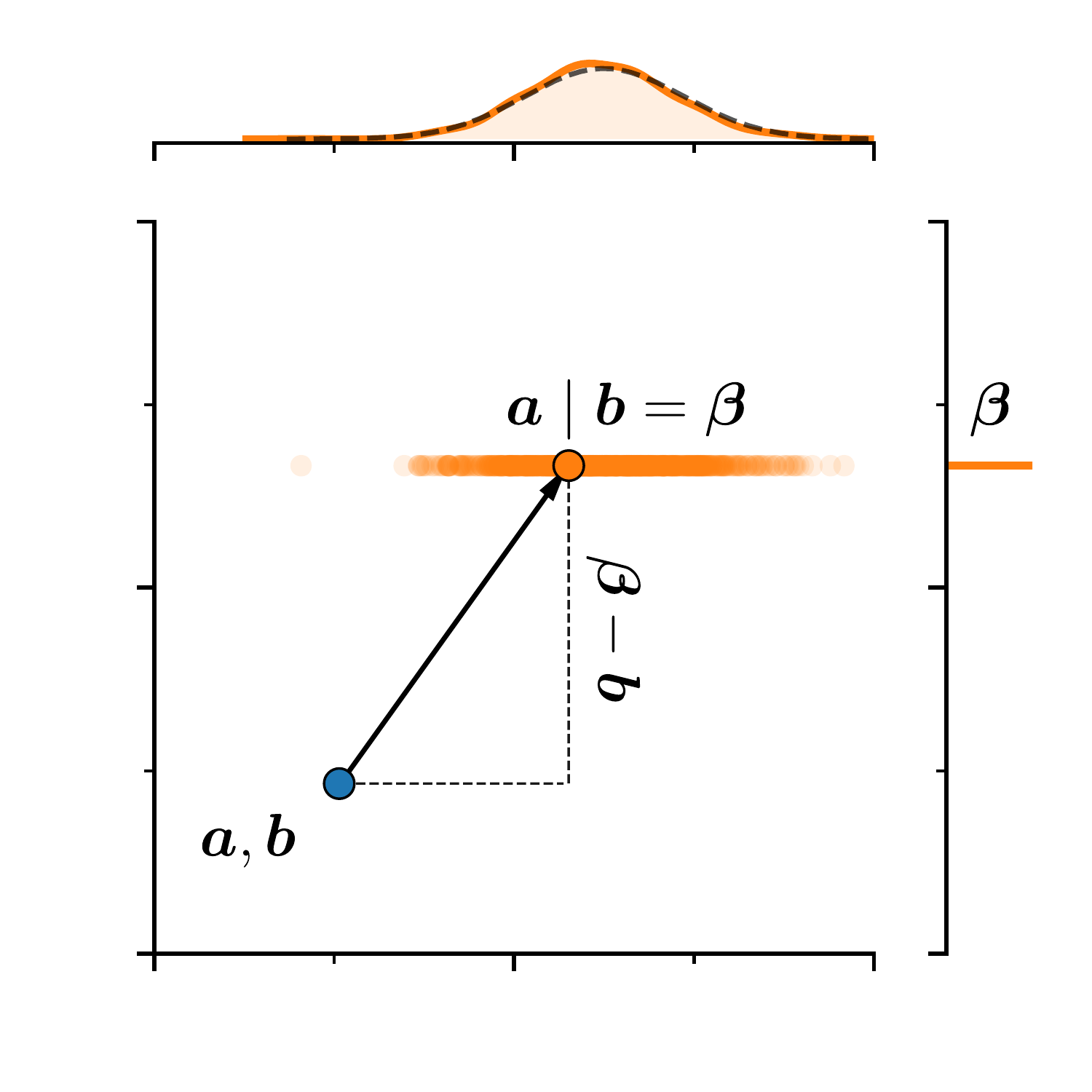}
    \end{minipage}
\caption{Visualization of Matheron's update rule for a bivariate normal distribution with correlation coefficient $\rho = 0.75$. \emph{Left:} Draws from $p(\v{a}, \v{b})$ are shown alongside the marginal distributions of $\v{a}$ and $\v{b}$. \emph{Right:} Theorem~\ref{thm:matheron_finite} is used to update samples shown on the left subject to the condition $\v{b} = \v{\beta}$. This process is illustrated in full for a particular draw. \emph{Top right:} the empirical distribution of the update samples is compared with $p(\v{a} \given \v{b} = \v{\beta})$.}
\label{fig:bvn_update}
\end{figure}

Instead of taking a \emph{distribution-first} stance on Gaussian conditionals, we may think of conditioning directly in terms of random variables. In this \emph{variable-first} paradigm, we will explicitly map samples from the prior to draws from a posterior and let the corresponding relationship between distributions follow implicitly. Throughout this work, we investigate this notion of \emph{pathwise conditioning} through the lens of the following result.

\begin{theorem}[Matheron's Update Rule]
\label{thm:matheron_finite}
Let $\v{a}$ and $\v{b}$ be jointly Gaussian, centered random variables. Then, the random variable $\v{a}$ conditional on $\v{b} = \v{\beta}$ may be expressed as
\[
\label{eqn:matheron_finite}
(\v{a} \given \v{b} = \v{\beta})
    \eqd
    \v{a} + \m{\Sigma}_{\v{a}, \v{b}}^{\vphantom{-1}} \m{\Sigma}_{\v{b}, \v{b}}^{-1}(\v{\beta} - \v{b}).
\]
\end{theorem}
\begin{proof}
Comparing the mean and covariance on both sides immediately affirms the result
\[
&\E\del[1]{\v{a} + \m\Sigma_{\v{a}, \v{b}}^{\vphantom{-1}}\m\Sigma_{\v{b}, \v{b}}^{-1}(\v{\beta} - \v{b})}
&
&\Cov\del[1]{\v{a} + \m\Sigma_{\v{a}, \v{b}}^{\vphantom{-1}}\m\Sigma_{\v{b}, \v{b}}^{-1}(\v{\beta} - \v{b})}
\nonumber \\
&\quad = \v{\mu}_{\v{a}}^{\vphantom{-1}} + \m\Sigma_{\v{a}, \v{b}}^{\vphantom{-1}}\m\Sigma_{\v{b}, \v{b}}^{-1}(\v{\beta} - \v{\mu}_{\v{b}}^{\vphantom{-1}})
&
&\quad= \m\Sigma_{\v{a},\v{a}}^{\vphantom{-1}} \!+\! \m\Sigma_{\v{a}, \v{b}}^{\vphantom{-1}} \m\Sigma_{\v{b}, \v{b}}^{-1} \m{\Sigma}_{\v{b},\v{b}}^{\vphantom{-1}} \m\Sigma_{\v{b}, \v{b}}^{-1} \m\Sigma_{\v{b}, \v{a}}^{\vphantom{-1}} \! - \! 2 \m\Sigma_{\v{a}, \v{b}}^{\vphantom{-1}} \m\Sigma_{\v{b}, \v{b}}^{-1} \m\Sigma_{\v{b}, \v{a}}^{\vphantom{-1}}
\\
&\quad=\E(\v{a} \given \v{b} = \v{\beta})
&
&\quad= 
\m\Sigma_{\v{a},\v{a}}^{\vphantom{-1}} \!-\! \m\Sigma_{\v{a}, \v{b}}^{\vphantom{-1}} \m\Sigma_{\v{b}, \v{b}}^{-1} \m\Sigma_{\v{b}, \v{a}}^{\vphantom{-1}} = \Cov(\v{a} \given \v{b} = \v{\beta}) \nonumber.
\]
\end{proof}
This observation leads to a straightforward, alternative recipe for generating $\v{a} \given \v{b} = \v{\beta}$: first, draw $\v{a}, \v{b} \sim p(\v{a}, \v{b})$; then, update this sample according to \eqref{eqn:matheron_finite}.
Compared to the location-scale approach discussed in Section~\ref{sec:mvn_conditioning_distributional}, a key difference is that we now sample \emph{before} conditioning, rather than after. Figure~\ref{fig:bvn_update} visualizes the deterministic process of updating previously generated draws from the prior subject to the condition $\v{b} = \v{\beta}$.

At first glance, Matheron's update rule may seem more like an interesting footnote than a valuable tool. 
Indeed, the conventional strategy for sampling $\v{a}, \v{b}$ (which requires us to take a matrix square root of $\m{\Sigma}$) is more expensive than that for generating $\v{a} \given \v{b} = \v{\beta}$. 
We will discuss this matter in detail in the later sections. For now, however, let us strengthen our intuition by delving deeper into this theorem's function-analytic origins.

\subsection{Deriving pathwise conditioning via conditional expectations}
\label{sec:conditional_expectation}

\newcommand{\Resid}{c}
\newcommand{\resid}{\varsigma}

Here, we overview the precise formalism that gives rise to the pathwise approach to conditioning Gaussian random variables and show how to \emph{derive} this result from first principles.
Throughout this section, we take $\v{a} \in \R^{m}$ and $\v{b} \in \R^{n}$ to be centered random vectors defined on the same probability space.

The core idea is to decompose $\v{a}$ as the sum of two independent terms---one that depends on $\v{b}$ and one that does not---and represent $\v{a}\given\v{b}=\v{\beta}$ by conditioning both terms on $\v{b}=\v{\beta}$. 
We first prove that conditioning this additive decomposition of $\v{a}$ is simple and intuitive.
\begin{lemma}
\label{lem:repr_cond}
Consider three random vectors $\v{a} \in \R^{m}$, $\v{b} \in \R^{n}$, $\v{\Resid} \in \R^{m}$ such that
\[ \label{eqn:cond_exp_repr}
\v{a} \eqd f(\v{b}) + \v{\Resid},
\]
where $f$ is a measurable function of $\v{b}$ and where $\v{b}$ is independent of $\v{\Resid}$. Then,
\[
\label{eqn:cond_exp_repr_pt2}
\del[1]{\v{a} \given \v{b} = \v{\beta}} \eqd f(\v{\beta}) + \v{\Resid}.
\]
\end{lemma}

\begin{proof}
Let $\pi_{\v{x}}$ denote the distribution of a generic random variable $\v{x}$. Further, let $\pi_{\v{a}\given\v{b}}(\. \given \.)$ be the (regular) conditional probability measure given by disintegration\footnote{See discussion and details on disintegration by \textcite{chang1997conditioning, kallenberg2006}.} of $(\v{a},\v{b})$, such that 
\[ 
\label{eqn:matheron_lemma_2}
\int_B
    \pi_{\v{a}\given\v{b}}(A \given\v\beta)
\d \pi_{\v{b}}(\v{\beta})
=
\P\del{\v{a} \in A, \v{b} \in B}
\]
for measurable sets $A \subseteq \R^{m}$, $B \subseteq \R^{n}$. When $\v{a} \given \v{b} = \v{\beta}$ is represented per \eqref{eqn:cond_exp_repr_pt2}, we have
\[
\label{eqn:matheron_lemma_1}
\begin{split}
\int_B
    \P\del{f(\v{\beta}) + \v{\Resid} \in A}
    \d \pi_{\v{b}}(\v{\beta})
&=
    \int_B
    \del{
        \int_{\R^{m}}
            \mathbbold{1}_{\cbr{f(\v{\beta}) + \v{\resid} \in A)}}
        \d \pi_{\v{\Resid}}(\v{\resid})
    }
    \d \pi_{\v{b}}(\v{\beta})
\\ 
    &= \int_{\R^{m} \x \R^{n}} \mathbbold{1}_{\cbr{f(\v{\beta}) + \v{\resid} \in A, \v{\beta} \in B)}}
    \d \pi_{\v{b}, \v{\Resid}}(\v{\beta}, \v{\resid})
\\ 
    &= \P\del{f(\v{b}) + \v{\Resid} \in A, \v{b} \in B}
    = \P\del{\v{a} \in A, \v{b} \in B},
\end{split}
\]
where we have begun by expressing probabilities as integrals of indicator functions, before using Tonelli's theorem and independence to express the iterated integral as the double integral over the joint probability measure $\pi_{\v{b}, \v{\Resid}}(\v{\beta}, \v{\resid})$.
Comparing the left-hand sides of \eqref{eqn:matheron_lemma_2} and \eqref{eqn:matheron_lemma_1} affirms the claim.
\end{proof}

In words, Lemma~\ref{lem:repr_cond} tells us that for suitably chosen functions $f$, the act of conditioning $\v{a}$ on $\v{b} = \v{\beta}$ amounts to adding an random variable $\v{\Resid}$ to a deterministic transformation $f(\v{\beta})$ of the outcome $\v{\beta}$.
For this statement to hold, we require the \emph{residual} $\v{\Resid} = \v{a} - f(\v{b})$ induced by $f$ to be independent of $\v{b}$. Fortunately, such a function $f$ is well-known in the special case of jointly Gaussian random variables---namely, the \emph{conditional expectation} $f : \v{b} \mapsto \E(\v{a}\given\v{b})$.

For square-integrable random variables, the conditional expectation of $\v{a}$ given $\v{b}$ is defined as the (almost surely) unique solution to the minimization problem
\[
\E(\v{a} \given \v{b})
&= 
    \argmin_{f \in \c{F}}
    \E \norm{\v{a} - f(\v{b})}^2,
\label{eqn:conditional_expectation_argmin}
\]
where $\c{F}$ denotes the set of all Borel-measurable functions $f:\R^{n}\to\R^{m}$ \cite[Chapter 6]{kallenberg2006}. 
Put simply, $\E(\v{a} \given \v{b})$ is the measurable function of $\v{b}$ that best predicts $\v{a}$ in the sense of minimizing the mean-square error \eqref{eqn:conditional_expectation_argmin}.
This characterization of the conditional expectation is equivalent to defining it as the orthogonal projection of $\v{a}$ onto the $\sigma$-algebra generated by $\v{b}$, denoted $\sigma(\v{b})$. Consequently, a necessary and sufficient condition for $\E(\v{a} \given \v{b}) \in \c{F}$ to uniquely solve \eqref{eqn:conditional_expectation_argmin} is that the residual $\v{\Resid} = \v{a} - \E(\v{a} \given \v{b})$ be orthogonal to all
$\sigma(\v{b})$-measurable random variables \cite[50]{luenberger1997optimization}. Here, \emph{orthogonality} can be understood as the absence of correlation, which (for jointly Gaussian random variables) implies independence. As a result, we may satisfy the assumptions of Lemma~\ref{lem:repr_cond} by writing
\[
    \v{a} = \E(\v{a} \given \v{b}) + \v{\Resid},
\]
such that $\v{a}$ decomposes into a function of $\v{b}$ and an independent variable $\v{\Resid} = \v{a} - \E(\v{a}\given\v{b})$.

As a final remark, we may also use these principles to concisely derive the conditional expectation for jointly Gaussian random variables. For now, suppose that the conditional expectation is a linear function of $\v{b}$, i.e. that $\E(\v{a} \given \v{b}) = \m{S} \v{b}$ for some matrix $\m{S} \in \R^{m \times n}$. 
To satisfy the orthogonality condition of \eqref{eqn:conditional_expectation_argmin}, we require $\Cov(\v{a}-\m{S}\v{b}, \v{b}) = \m{0}$, implying that $\m{\Sigma}_{\v{a},\v{b}} - \m{S}\m{\Sigma}_{\v{b},\v{b}} = \m{0}$. 
Rearranging terms and solving for $\m{S}$ gives $\m{S} = \m{\Sigma}_{\v{a},\v{b}}^{\vphantom{-1}}\m{\Sigma}_{\v{b},\v{b}}^{-1}\v{b}$.
With this expression in hand, to show that linearity was assumed without loss of generality, write $\v{a} = \m{S} \v{b} + \v{a} - \m{S} \v{b}$, which we may express as as $\v{a} = \m{S} \v{b} + \v{\Resid}$.
Taking the conditional expectation of both sides, we may directly calculate $\E(\v{a} \given \v{b})$ by writing
\[
\E(\v{a} \given \v{b}) 
=
    \E(\m{S} \v{b} + \v{\Resid} \given \v{b})
=
    \underbracket[0.5pt]{\E(\m{S} \v{b} \given \v{b})}_{\smash{\m{S}\v{b}}} + \underbracket[0.5pt]{\E(\v{\Resid})}_{\smash{\v{0}}}
=
    \m{\Sigma}_{\v{a},\v{b}}^{\vphantom{-1}}\m{\Sigma}_{\v{b},\v{b}}^{-1}\v{b},
\]
where we have used linearity of conditional expectation, followed by independence of $\v{\Resid}$ and $\v{b}$ to go from the second to the third expression.
We now revisit Theorem~\ref{thm:matheron_finite}.

{
\renewcommand{\thetheorem}{1}
\begin{theorem}[Matheron's Update Rule]
Let $\v{a}$ and $\v{b}$ be jointly Gaussian, centered random vectors. Then, the random vector $\v{a}$ conditional on $\v{b} = \v{\beta}$ may be expressed as
\[
(\v{a} \given \v{b} = \v{\beta})
    \eqd
    \v{a} + \m{\Sigma}_{\v{a},\v{b}}^{\vphantom{-1}}\m{\Sigma}_{\v{b},\v{b}}^{-1}(\v{\beta} - \v{b}).
\tag{\ref{eqn:matheron_finite}}
\]
\end{theorem}
}
\begin{proof}
With $\smash{\v{\Resid} = \v{a} - \m{\Sigma}_{\v{a},\v{b}}^{\vphantom{-1}}\m{\Sigma}_{\v{b},\v{b}}^{-1}\v{b}}$, begin by writing
\[
\v{a}
=
\E(\v{a} \given \v{b}) + \del{\v{a} - \E(\v{a} \given \v{b})}
=
\smash{
\m{\Sigma}_{\v{a},\v{b}}^{\vphantom{-1}}\m{\Sigma}_{\v{b},\v{b}}^{-1}\v{\v{b}}
+
\v{\Resid}.
}
\]
Since $\v{b}$ and $\v{\Resid}$ are jointly Gaussian but uncorrelated, it follows that they are independent. 
Setting $f(\v{b}) = \m{\Sigma}_{\v{a},\v{b}}^{\vphantom{-1}}\m{\Sigma}_{\v{b},\v{b}}^{-1}\v{\v{b}}$ and using Lemma~\ref{lem:repr_cond} to condition both sides on $\v{b} = \v{\beta}$ gives 
\[
(\v{a} \given\v{b} = \v{\beta})
\eqd
    \m{\Sigma}_{\v{a},\v{b}}^{\vphantom{-1}}\m{\Sigma}_{\v{b},\v{b}}^{-1}\v{\v{\beta}}
    + \del[1]{\v{a} - \m{\Sigma}_{\v{a},\v{b}}^{\vphantom{-1}}\m{\Sigma}_{\v{b},\v{b}}^{-1}\v{b}}
= 
    \v{a} + \m{\Sigma}_{\v{a},\v{b}}^{\vphantom{-1}}\m{\Sigma}_{\v{b},\v{b}}^{-1}(\v{\beta} - \v{b}).
\]
Hence, the claim follows.
\end{proof}

In summary, we have shown that Matheron's update rule (Theorem~\ref{thm:matheron_finite}) is a direct consequence of the fact that a Gaussian random variable $\v{a}$ conditioned on the outcome $\v{\beta}$ of another (jointly) Gaussian random variable $\v{b}$ may be expressed as the sum of two independent terms: the conditional expectation $\E(\v{a} \given \v{b} = \v{\beta})$ evaluated at $\v{\beta}$ and the residual $\v{\Resid} = \v{a} - \E(\v{a} \given \v{b})$.
Rearranging these terms gives \eqref{eqn:matheron_finite}.  

With these ideas in mind, we are now ready to explore this work's primary theme: Matheron's update rule enables us to decompose $\v{a} \given \v{b} = \v{\beta}$ into the prior random variable $\v{a}$ and a data-driven update $\m{\Sigma}_{\v{a},\v{b}}^{\vphantom{-1}}\m{\Sigma}_{\v{b},\v{b}}^{-1}(\v{\beta} - \v{b})$ that explicitly corrects for the error in the coinciding value of $\v{b}$ given the condition $\v{b} = \v{\beta}$. 
Hence, Theorem~\ref{thm:matheron_finite} provides an explicit means of separating out the influence of the prior from that of the data. We now proceed to investigate the implications of pathwise conditioning for Gaussian processes.

\section{Conditioning Gaussian processes and random functions}
\label{sec:gp_conditioning}

A Gaussian process (GP) is a random function $f : \c{X} \to \R$, such that, for any finite collection of points $\m{X} \subset \c{X}$, the random vector $\rv{f} = f(\m{X})$ follows a Gaussian distribution. Such a process is uniquely identified by a mean function $\mu : \c{X} \to \R$ and a positive semi-definite kernel $k: \c{X} \times \c{X} \to \R$. Hence, if $f \sim \c{GP}(\mu, k)$, then $\rv{f} \sim \c{N}(\v{\mu}, \m{K})$ is multivariate normal with mean $\v{\mu} = \mu(\m{X})$ and covariance $\m{K} = k(\m{X}, \m{X})$.

Throughout this section, we investigate different ways of reasoning about the random variable $\v{f}_{*} \given \v{f}_{n} = \v{y}$ for some non-trivial partition $\v{f} = \v{f}_{n} \oplus \v{f}_{*}$. Here, $\v{f}_{n} = f(\m{X}_{n})$ are process values at a set of training locations $\m{X}_{n} \subset \m{X}$ where we would like to introduce a condition $\v{f}_{n} = \v{y}$, while $\v{f}_{*} = f(\m{X}_{*})$ are process values at a set of test locations $\m{X}_{*} \subset \m{X}$ where we would like to obtain a random variable $\v{f}_{*} \given \v{f}_{n} = \v{y}$. Mirroring Section~\ref{sec:mvn_conditioning}, we begin by reviewing distributional conditioning, before examining its pathwise counterpart.

\subsection{Distributional conditioning}
\label{sec:gp_conditioning_distributional}
As in finite-dimensional cases, we may obtain $\v{f}_{*} \given \v{y}$ by first finding its conditional distribution. Since process values $(\v{f}_{n}, \v{f}_{*})$ are defined as jointly Gaussian, this procedure closely resembles that of Section~\ref{sec:mvn_conditioning_distributional}: we factor out the marginal distribution of $\v{f}_{n}$ from the joint distribution $p(\v{f}_{n}, \v{f}_{*})$ and, upon canceling, identify the remaining distribution as $p(\v{f}_{*} \given \v{y})$. Having done so, we find that the conditional distribution is the Gaussian $\c{N}(\v{\mu}_{* \given \v{y}}, \m{K}_{*,* \given \v{y}})$ with moments
\[
\label{eqn:gp_cond_moments}
\v{\mu}_{* \given \v{y}} &= \v{\mu}_{*} + \m{K}_{*,n}^{\vphantom{-1}}\m{K}_{n,n}^{-1}(\v{y} - \v{\mu}_{n})
&
\m{K}_{*,* \given \v{y}}^{\vphantom{-1}} &= 
    \m{K}_{*,*} 
    - \m{K}_{*, n}^{\vphantom{-1}}\m{K}_{n, n}^{-1}\m{K}_{n, *}^{\vphantom{-1}}
.
\]
As before, we may now generate $\v{f}_{*} \given \v{y}$ in $\c{O}(*^{3})$ time using a location-scale transform \eqref{eqn:location_scale}.

This strategy for sampling Gaussian process posteriors is subtly different from the one given in Section~\ref{sec:mvn_conditioning_distributional}. A Gaussian process is a random function, and conditioning on $\v{f}_{n} = \v{y}$ does not change this fact. Unfortunately, (conditional) distributions over infinite-dimensional objects can be difficult to manipulate in practice. Distributional approaches, therefore, focus on finite-dimensional subsets $\v{f} = \v{f}_{n} \oplus \v{f}_{*}$, while marginalizing out the remaining process values. Doing so allows them to perfectly describe the random variable $\v{f}_{*} \given \v{y}$ via its mean and covariance \eqref{eqn:gp_cond_moments}.

When it comes to sampling $\v{f}_{*} \given \v{y}$, however, these approaches have clear limitations. As discussed previously, a key issue is that their $\c{O}(*^{3})$ time complexity restricts them to problems that only require us to jointly simulate process values at a manageable number of test locations (up to several thousand). In some senses, this condition is fairly generous. After all, we are often only asked to generate a handful of process values at a time. Still, other problems effectively require us to realize $f \given \v{y}$ in its entirety. Similar issues arise when $\m{X}_{*}$ is not defined in advance, such as when gradient information is used to adaptively determine the locations at which to jointly sample the posterior. In these cases and more, we would ideally like to sample actual functions that we can efficiently evaluate and automatically differentiate at arbitrary test locations. To this end, we now examine the direct approach to conditioning draws of $f \sim \c{GP}(\mu, k)$.

\subsection{Pathwise Conditioning}
\label{sec:gp_conditioning_pathwise}

\begin{figure}
    \centering
    \includegraphics[width=\textwidth]{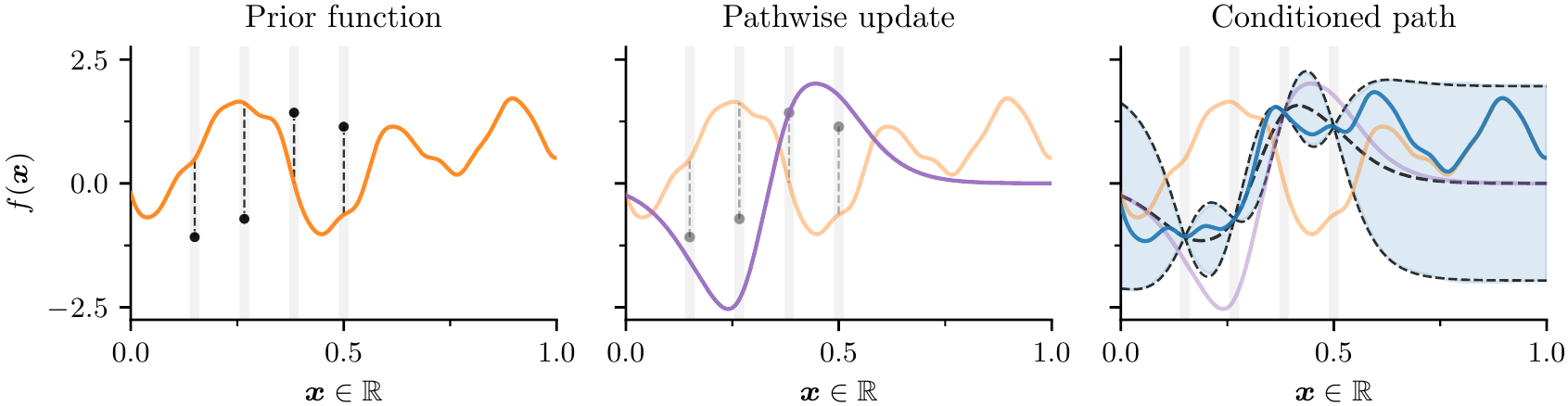}
    \caption{Visual guide for pathwise conditioning of Gaussian processes.
    \emph{Left:} The residual $\v{y} - \v{f}_{n}$ (dashed black) of a draw $f \sim \c{GP}(0, k)$, shown in orange, given observations $\v{y}$ (black).
    \emph{Middle:} A pathwise update (purple) is constructed in accordance with Corollary~\ref{cor:matheron_inf}.
    \emph{Right:} Prior and update are combined to represent conditional (blue). Empirical moments (light blue) of $10^5$ conditioned paths are compared with those of the model (dashed black). The sample average, which matches the posterior mean, has been omitted for clarity.}
    \label{fig:pathwiseUpdate_walkthrough}
\end{figure}

Examining the pathwise update given by Theorem~\ref{thm:matheron_finite}, it is natural to suspect that an analogous statement holds for Gaussian processes. A quick check confirms this hypothesis.

\begin{corollary}
\label{cor:matheron_inf}
For a Gaussian process $f \~ \c{GP}(\mu, k)$ with marginal $\rv{f}_{n} = f(\m{X}_n)$, the process conditioned on $\v{f}_n = \v{y}$ may be expressed as
\[
    \label{eqn:matheron_inf}
    \underbracket[0.5pt]{(f \given \v{y})(\.)\vphantom{\m{K}_{n,n}^{-1}}}_{\f{conditional}}
\overset{\d}{=}
    \underbracket[0.5pt]{f(\.) \vphantom{\m{K}_{n,n}^{-1}}}_{\f{prior}} 
    + 
    \underbracket[0.5pt]{k(\., \m{X}_{n})\m{K}_{n,n}^{-1}(\v{y} - \v{f}_{n})}_{\f{update}}.
\]
\end{corollary}
\begin{proof}
Follows by applying Theorem \ref{thm:matheron_finite} to an arbitrary set of locations.
\end{proof}
Figure~\ref{fig:pathwiseUpdate_walkthrough} acts a visual guide to Corollary~\ref{cor:matheron_inf}. From left to right, we begin by generating a realization of $f \sim \c{GP}(\mu, k)$ using methods that will soon be introduced in Section~\ref{sec:priors}. Having obtained a sample path, we then use the pathwise update \eqref{eqn:matheron_inf} to define a function $k(\., \m{X}_{n})\m{K}^{-1}_{n,n}(\v{y} - \v{f}_{n})$ to account for the residual $\v{y} - \v{f}_{n}$.
Adding these two functions together produces a draw from a GP posterior, the behavior of which is shown on the right.
Whereas distributionally conditioning on $\v{f}_{n} = \v{y}$ in \eqref{eqn:gp_cond_moments} tells us how the GP's statistic properties change, pathwise conditioning \eqref{eqn:matheron_inf} tells us what happens to individual sample paths.
This paradigm shift echoes the running theme: Gaussian (process) conditionals can be directly viewed in terms of random variables. The power of Corollary~\ref{cor:matheron_inf} is that it impacts \emph{how} we think about Gaussian process posteriors and, therefore, \emph{what} we do with them.

Having said this, there are several hurdles that we must overcome in order to use the pathwise update \eqref{eqn:matheron_inf} in the real world. 
First, we are typically unable to practically sample functions $f \sim \c{GP}(\mu, k)$ from (non-degenerate) Gaussian process priors exactly. 
A Gaussian process can generally be written as a linear combination of elementary \emph{basis functions}. 
When the requisite number of basis functions is infinite, however, evaluating this linear combination is usually impossible. In Section~\ref{sec:priors}, we will therefore investigate different ways of approximating $f(\.)$ using a finite number of operations.

Second, we incur $\c{O}(n^{3})$ time complexity when na\"ively carrying out \eqref{eqn:matheron_inf}, due to the need to solve the linear system of equations $\m{K}_{n,n}\v{v} = \v{y} - \v{f}_{n}$ for a vector $\v{v} \in \mathbb{R}^{n}$ such that
\[
\label{eqn:pathwise_update_canonical}
(f \given \v{y})(\.)
\eqd 
    f(\.) + 
    \underbracket[0.5pt]{
        \sum_{i=1}^{n} v_{i} k(\., \v{x}_{i})
    }_{\mathclap{n\t{-dimensional basis}}}
.
\]
Here, we have re-expressed the matrix-vector product in \eqref{eqn:matheron_inf} as an expansion with respect to the canonical basis functions $k(\., \v{x}_{i})$ centered at training locations $\v{x}_{i} \in \m{X}_{n}$.
For large training sets $\del{\v{x}_i, y_i}_{i=1}^{n}$, direct application of \eqref{eqn:matheron_inf} may prove prohibitively expensive. 
By the same token, the stated pathwise update does not hold when outcomes $\v{y}$ are not defined as realizations of process values $\v{f}_{n}$. 
In Section~\ref{sec:updates}, we will consider various means of resolving these challenges and ones like them.

\subsection{Historical remarks}
\label{sec:historical_remarks}
Prior to continuing, we pause to reflect on the historical developments that have paved the way for this work. In a 2005 tribute to geostatistics pioneer Georges Matheron, \textcite{chiles2005prediction} comment that
\begin{quote}
[Matheron's update rule] is nowhere to be found in Matheron's entire published works, as he merely regarded it as an immediate consequence of the orthogonality of the [conditional expectation] and the [residual process].
\end{quote}
As if to echo this very sentiment, \textcite{doucet10} begins a much appreciated technical note on the subject of Theorem~\ref{thm:matheron_finite} with the remark
\begin{quote}
This note contains no original material and will never be submitted anywhere for publication. However it might be of interest to people working with [Gaussian processes] so I am making it publicly available.
\end{quote}
The presiding opinion, therefore, seems to be that Matheron's update rule is \emph{too simple} to warrant extended study. Indeed, Theorem~\ref{thm:matheron_finite} is exceedingly straightforward to verify. As is often the case, however, this result is harder to discover if one is not already aware of its existence. This dilemma may help to explain why Matheron's update rule is absent from standard machine learning texts. By deriving this result from first principles in Section~\ref{sec:conditional_expectation}, we hope to encourage fellow researchers to explore the strengths (and weaknesses) of the pathwise viewpoint espoused here.

We are not the first to have realized the practical implications of pathwise conditioning for GPs. Corollary~\ref{cor:matheron_inf} is relatively well-known in geostatistics \cite{journel1978mining, defouquet94, emery07, chiles12}. Similarly, \textcite{oliver1996conditional} discusses Matheron's update rule for Gaussian likelihoods (Section~\ref{sec:updates_gaussian}). Along the same lines, closely related ideas were rediscovered in the 1990s with applications to astrophysics. In particular, \textcite{hoffman91} propose the use of spectral approximations to stationary priors (Section~\ref{sec:priors_stationary}) in conjunction with canonical pathwise updates \eqref{eqn:pathwise_update_canonical}.

Nevertheless, these formulae are seldom seen in machine learning. We hope to systematically organize these findings (along with our own) and communicate them to a general audience of theorists and practitioners alike. The following sections therefore catalog various notable approaches to representing Gaussian process priors and pathwise updates.

\section{Sampling functions from Gaussian process priors}
\label{sec:priors}

The pathwise representation of GP posteriors described in the Section~\ref{sec:gp_conditioning_pathwise} allows us to represent $f \given \v{y}$ by transforming a draw of $f \sim \c{GP}(0, k)$. When interpreted as a generative strategy, this approach to sampling can only be deemed \emph{efficient} if the tasks of realizing the prior and performing the update both scale favorably in the total number of locations $\vert\m{X}\vert = \vert\m{X}_{n}\vert + \vert\m{X}_{*}\vert$. Half of the battle is, therefore, to obtain faithful but affordable draws of~$f$.
Fortunately, GP priors often exhibit convenient mathematical properties not present in their posteriors, which can be utilized to sample them efficiently.

We focus on methods for generating random \emph{functions} that we may evaluate at arbitrary locations $\v{x} \in \c{X}$ in $\c{O}(1)$ time and whose marginal distributions approximate those of $f \sim \c{GP}(0, k)$. Conceptually, techniques discussed throughout this section will approximate GP priors as 
random linear combinations of suitably chosen basis functions $\v{\phi} = (\phi_{1}, \ldots, \phi_{\ell})$. Specifically, we will focus on \emph{Bayesian linear models} with Gaussian random weights
\[
\label{eqn:bayesian_linear_model}
\tilde{f}(\.) &= \sum_{i=1}^\ell w_i \phi_i(\.)
&
\v{w} &\sim \c{N}(\v{0}, \m{\Sigma}_{\v{w}}),
\]
where the covariance of weights $\v{w}$ will vary by case. Notice that, for any finite collection of points $\m{X} \subset \c{X}$, the random vector $\v{\tilde{f}} = \tilde{f}(\m{X})$ follows the Gaussian distribution $\c{N}(\v{0}, \m{\Phi}^{\vphantom{\top}} \m{\Sigma}_{\v{w}} \m{\Phi}^\top)$, where $\m{\Phi} = \v{\phi}(\m{X})$ is a $\vert \m{X} \vert \times \ell$ matrix of features. By design then, $\tilde{f}$ is a Gaussian process. \textcite{rasmussen06} refer to \eqref{eqn:bayesian_linear_model}  as the \emph{weight-space} view of GPs.

From this perspective, the task of efficiently sampling the prior $\tilde{f}$ reduces to one of generating random weights $\v{w}$. In practice, $\m{\Sigma}_{\v{w}}$ is typically diagonal, thereby enabling us to sample $\tilde{f}$ in $\c{O}(\ell)$ time. We stress that, for any draw of $\v{w}$, the corresponding realization of $\tilde{f}$ is simply a deterministic function. In particular, we incur $\c{O}(1)$ cost for evaluating $\tilde{f}(\v{x})$ and may readily differentiate this term with respect to $\v{x}$ (or other parameters of interest).

Below, we review popular strategies for obtaining Bayesian linear models such that $\tilde{f} \overset{\d}{\approx} f$. 
Our presentation is intended to communicate different angles for attacking this problem and is by no means exhaustive. To set the scene for these approaches, we begin by recounting some properties of the gold standard: location-scale methods.

\subsection{Location-scale transformations}
\label{sec:priors_exact}

Location-scale methods $\eqref{eqn:location_scale}$ are the most widely used approach for generating Gaussian random vectors. These generative strategies are \emph{exact} (up to machine precision).
Given locations $\m{X}$, we may simulate $\v{f} = f(\m{X})$ in location-scale fashion
\[
\label{eqn:location_scale_gp}
f(\m{X}) &\eqd \m{K}^{\nicefrac{1}{2}} \v{\zeta}
&
\v{\zeta} &\~\c{N}(\v{0},\m{I})
\]
by multiplying a square root covariance matrix $\m{K}^{\nicefrac{1}{2}}$ by a standard normal vector $\v{\zeta}$.

While \eqref{eqn:location_scale_gp} rightfully stands as the method of choice for many problems, it is not without shortcoming. Chief among these issues is the fact that algorithms for obtaining a matrix square root of $\m{K}$ scale cubically in $\vert \m{X} \vert$. In most cases, this limits the use of location-scale approaches to cases where the length of the desired Gaussian random vector is manageable (up to several thousand). 
This overhead can be interpreted to mean that we incur $\c{O}(i^{2})$ cost for realizing the $i$-th element of $\v{f}$, which leads us to our second issue: reusing a draw of $\v{f}_{n}$ to efficiently generate the remainder of $\v{f} = \v{f}_{n} \oplus \v{f}_{*}$ requires us to sample from the conditional distribution
\[
\label{eqn:conditional_iterated}
\v{f}_{*} \given \v{f}_{n}
\sim 
    \c{N}\del{
        \v{\mu}_{*} + \m{K}_{*,n} \m{K}_{n,n}^{-1}(\v{f}_{n} - \v{\mu}_{n}),
        \m{K}_{*,*} - \m{K}_{*,n} \m{K}_{n,n}^{-1} \m{K}_{n,*}
    }.
\]
Despite matching asymptotic costs, iterative approaches to sampling $\v{f}$ are substantially slower than simultaneous ones. In applied settings, however, test locations $\m{X}_{*}$ are often determined adaptively, forcing location-scale-based methods for generating $\v{f}$ to repeatedly compute \eqref{eqn:conditional_iterated}. Further refining this predicament, we arrive at a final challenge: pathwise derivatives.

Differentiation is a linear operation. The gradient of a Gaussian process $f$ with respect to a location $\v{x}$ is, therefore, another Gaussian process $f^\prime$.
By construction, these GPs are correlated. Using gradient information to maneuver along a sample path---for example, to identify its extrema---therefore requires us to re-condition both processes on the realized values of $f(\v{x})$ and $f^\prime(\v{x})$ at each successive step of gradient descent.

Prior to continuing, it is worth noting that the limitations of location-scale methods can be avoided in certain cases. In particular, the otherwise cubic costs for computing a square root in $\eqref{eqn:location_scale_gp}$ can be dramatically reduced by exploiting structural assumptions regarding covariance matrices $\m{K}$. 
Well-known examples of structured matrices include banded and sparse ones in the context of one-dimensional Gaussian processes and Gauss--Markov random fields \cite{rue2005gaussian, durrande19, loper20}, block-Toeplitz Toeplitz-block ones when evaluating stationary product kernels on regularly-spaced grids $\m{X} \subset \c{X}$ \cite{zimmerman1989computationally, wood1994simulation, dietrich97}, and kernel-interpolation-based ones \cite{wilson2015kernel, pleiss2018constant}. When the task at hand permits their usage, these methods are highly effective.

The following sections survey different approaches to overcoming the challenges put forth above by approximating Gaussian process priors as finite-dimensional Bayesian linear models.

\subsection{Stationary covariances}
\label{sec:priors_stationary}
Stationary covariance functions $k(\v{x}, \v{x}') = k(\v{x} - \v{x}')$, such as the Mat\'{e}rn family's limiting squared exponential kernel, give rise to a significant portion of GP priors in use today.
For centered priors $f \sim \c{GP}(0, k)$, stationarity encodes the belief that the relationship between process values $f(\v{x}_{i})$ and $f(\v{x}_{j})$ is solely determined by the difference $\v{x}_{i}-\v{x}_{j}$ between locations $\v{x}_{i}$ and $\v{x}_{j}$. 
Simple but expressive, stationarity is the go-to modeling assumption in many applied settings.

These kernels exhibit a variety of special properties that greatly facilitate the construction of efficient, approximate priors. Here, we restrict attention to kernels admitting a spectral density $\rho$, and focus on the class of estimators formed by discretizing the spectral representation of $k$
\[
\label{eqn:wiener_khinchin}
k(\v{x} - \v{x}^\prime) &= \int_{\R^d} e^{2 \pi i \v{\omega}^\top (\v{x} - \v{x}^\prime)} \rho(\v\omega) \d \v{\omega}
&
\rho(\v{\omega}) &= \int_{\R^d} e^{- 2 \pi i \v{\omega}^\top \v{x}}k(\v{x}) \d\v{x}
.
\]
By the \emph{kernel trick} \cite{scholkopf01}, a kernel $k$ can be written as the inner product in a corresponding reproducing kernel Hilbert space (RKHS) $\c{H}_{k}$ equipped with a feature map $\varphi: \c{X} \-> \c{H}_{k}$. 
In many cases, this inner product can be approximated by
\[
\label{eqn:rkhs_inner_prod_onb}
k(\v{x}, \v{x}')
=
\innerprod{\varphi(\v{x})}{\varphi(\v{x}')}_{\c{H}_{k}}
\approx
\v{\phi}(\v{x})^\top\ \overline{\v{\phi}(\v{x}')},
\]
where $\v{\phi} : \c{X} \-> \mathbb{C}^\ell$ is some finite-dimensional feature map and $\overline{\v{\phi}(\v{x}')}$ denotes the complex conjugate.
Based on this idea, the method of \emph{random Fourier features} \cite{rahimi08} constructs a Monte Carlo estimate to a stationary kernel by representing the right-hand side of \eqref{eqn:rkhs_inner_prod_onb} with $\ell$ complex exponential basis functions $\phi_{j}(\v{x}) = \ell^{\nicefrac{-1}{2}} \exp(2 \pi i\v{\omega}_{j}^\top \v{x})$, whose parameters $\v{\omega}_{j}$ are sampled proportional to the corresponding spectral density $\rho(\v{\omega}_{j})$.\footnote{Using elementary trigonometric identities, we may also derive a related family of basis functions $\phi : \c{X} \to \R^{\ell}$ with $\phi_{j}(\v{x}) = \sqrt{\nicefrac{2}{\ell}} \cos(2 \pi \v{\omega}_{j}^\top\v{x} + \tau_{j})$, where $\tau_j \sim \c{U}(0, 2 \pi)$.}

Given an $\ell$-dimensional basis $\v{\phi} = (\phi_{1}, \ldots, \phi_{\ell})$, we may now proceed to approximate the true prior according to the Bayesian linear model
\[
\label{eqn:bayesian_linear_model_rff}
\tilde{f}(\.) &= \sum_{i=1}^\ell w_i \phi_i(\.)
&
w_i &\~\c{N}(0,1).
\]
Under this approximation, $\tilde{f}$ is a random function satisfying $\v{\tilde{f}}_{n} \sim \c{N}(\v{0}, \m{\Phi}_{n}^{\vphantom{\top}}\m{\Phi}_{n}^\top)$, where $\m{\Phi}_{n} = \v{\phi}(\m{X}_{n})$ is an $n \times \ell$ matrix of features. Per the beginning of this section, then, $\tilde{f}$ is a Gaussian process whose covariance approximates that of $f$.

The random Fourier feature approach is particularly appealing since its position as a Monte Carlo estimator implies that the error introduced by the $\ell$-dimensional basis $\v{\phi}$ decays at the \emph{dimension-free} rate $\ell^{\nicefrac{-1}{2}}$ \cite{sutherland15}. 
This property enables us to balance accuracy and cost by choosing $\ell$ to suite the task at hand.

\subsection{Karhunen--Lo\`{e}ve expansions}
\label{sec:priors_kl}
While exploitation of stationarity is arguably the most common route when constructing approximate priors, it is neither unique nor optimal. A powerful alternative is to utilize the \emph{Karhunen--Lo\`{e}ve expansion} of a Gaussian process prior \cite{castro1986principal, fukunaga2013introduction}.

We begin by considering the family of $\ell$-dimensional Bayesian linear models $\tilde{f}(\.) = \v{\phi}(\.)^\top\v{w}$ consisting of orthonormal basis functions $\phi_{i}: \c{X} \to \mathbb{R}$ on a compact space $\c{X}$. Following standard theory \cite{fukunaga2013introduction}, the \emph{optimal} $\tilde{f}$ for approximating a Gaussian process $f$ (in the sense of minimizing mean square error) is found by truncating its Karhunen--Lo\`{e}ve expansion
\[
\label{eqn:kl_expansion}
f(\.) &= \sum_{i=1}^\infty w_i \phi_i(\.)
&
w_i &\~\c{N}(0,\lambda_i),
\]
where $\phi_i$ and $\lambda_{i}$ are, respectively, the $i$-th eigenfunction and eigenvalue of the covariance operator $\psi \mapsto \int_{\c{X}} \psi(\v{x})k(\v{x},\.) \d \v{x}$, written in decreasing order of $\lambda_{i}$.\footnote{These eigenvalues are well-ordered and countable as consequence of the compactness of $\c{X}$.} Truncated versions of these expansions are used as both bases for constructing optimal approximate GPs \cite{zhu1997gaussian, solin2020hilbert} and modeling tools in their own right \cite{krainski19}. 
Depending on the case, eigenfunctions $\phi_{i}$ are either derived from first principles \cite{krainski19} or obtained by numerical methods \cite{lindgren11,lord14,solin2019know}.

In addition to being optimal, Karhunen--Lo\`{e}ve expansions are exceedingly general. Even when a covariance function $k$ is non-stationary or the domain $\c{X}$ is non-Euclidean---such as when Gaussian processes are used to represent functions on manifolds \cite{borovitskiy2020matern} and graphs \cite{borovitskiy2020graph}---the Karhunen--Lo\`{e}ve expansion often exists. 

Widespread use of truncated eigensystems is largely impeded by their frequent lack of convenient, analytic forms. This issue is compounded by the fact that efficient, numerical methods for obtaining \eqref{eqn:kl_expansion} typically require us to manipulate bespoke mathematical properties of specific kernels. These properties are often closely related to the differential-equation-based perspectives of Gaussian processes introduced in the following section.

\subsection{Stochastic partial differential equations}
\label{sec:priors_spde}

Many Gaussian process priors, such as the Mat\'{e}rn family, can be expressed as solutions of \emph{stochastic partial differential equations} (SPDEs). 
SPDEs are common in fields such as physics, where they describe natural phenomena (such as diffusion and heat transfer); many of which share a deep connection with the squared exponential kernel \cite{grigoryan2009heat}.
Additionally, SPDEs are often the starting point when designing non-stationary GP priors \cite{krainski19}. Below, we detail how the \emph{Galerkin finite element method} \cite{evans10, lindgren11,lord14} can be used to construct Bayesian linear models that approximate GP priors capable of being represented as SPDEs.

Suppose a Gaussian process $f \sim \c{GP}(0, k)$ satisfies $\c{L} f = \c{W}$, where $\c{L}$ is a linear differential operator and $\c{W}$ is a Gaussian white noise process \cite{lifshits2012lectures}. 
Here, we demonstrate how to derive a Gaussian process $\tilde{f}$ that approximately satisfies this SPDE.
To begin, we express $\c{L}f = \c{W}$ in its weak form\footnote{One typically integrates $(\c{L}f)(\v{x})g(\v{x})$ by parts, either by necessity or due to affordances of the basis $\phi_i$. We suppress this to ease notation.}
\[
\label{eqn:gp_spde}
\int_{\c{X}} (\c{L}f)(\v{x}) g(\v{x}) \d\v{x} = \int_{\c{X}} g(\v{x}) \d\c{W}(\v{x}),
\]
where $g$ is an arbitrary element of an appropriate class of test functions. Next, we proceed by approximating both the desired solution $f$ and the test function $g$ with respect to a finite-dimensional basis as $\tilde{f}(\.) = \sum_{i=1}^\ell w_i \phi_i(\.)$ and $\tilde{g}(\.) = \sum_{j=1}^\ell v_j \phi_j(\.)$.
Substituting these terms into \eqref{eqn:gp_spde} and differentiating both sides with respect to the coefficients of $\tilde{g}$, we obtain the following expression for each $j = 1,\ldots, \ell$
\[
\sum_{i=1}^\ell w_i \underbracket[0.5pt]{\int_{\c{X}} (\c{L} \phi_i)(\v{x}) \phi_j(\v{x})\d\v{x}}_{A_{ij}} 
= 
\underbracket[0.5pt]{\int_{\c{X}} \phi_j(\v{x}) \d\c{W}(\v{x})}_{b_j}
.
\]
Defining $\m{M} = \Cov(\v{b})$, where $\Cov(b_{i},b_{j}) = \innerprod{\phi_i}{\phi_j}$ coincides with the finite-element mass matrix, allows us to rearrange this system of random linear equations in matrix-vector form by writing $\m{A}\v{w} = \v{b}$.
The basis coefficients of the random function $\tilde{f}$ are, therefore, distributed as $\v{w} \sim \c{N}\del{\v{0}, \m{A}^{-1}\m{M}\m{A}^{-\top}}$. 
As in the previous sections, $\tilde{f}$ can be seen as the weight-space view of a corresponding Gaussian process.

A popular choice is to employ compactly supported basis functions $\phi_i$ \cite{lindgren11}. The matrices $\m{A}$ and $\m{M}$ are then sparse, and the resulting linear systems can be solved efficiently. 
For example, the family of piecewise linear basis functions is a simple but effective choice for second order differential operators $\c{L}$ \cite{evans10,lord14}.\footnote{A second order differential operator gives rise to a first-order bilinear form when integrated by parts, which matches with piecewise linear basis functions which are once differentiable almost everywhere. For higher-order operators, a piecewise polynomial basis may be used instead.}

\subsection{Discussion}
\label{sec:priors_discussion}

This section has focused on identifying finite-dimensional bases with which to construct Bayesian linear models $\tilde{f}(\.) = \v{\phi}(\.)^\top\v{w}$. These model can be seen as \emph{weight-space} interpretations \cite{rasmussen06} of corresponding Gaussian process priors $\tilde{f} \sim \c{GP}(0, \tilde{k})$ with covariance functions $\tilde{k}(\v{x}, \v{x}^\prime) = \v{\phi}(\v{x})^\top \m{\Sigma}_{\v{w}} \v{\phi}(\v{x}^\prime)$. Since $\v{w}$ and $\v{\tilde{f}}_{n} = \m{\Phi}_{n} \v{w}$ are jointly normal, Theorem~\ref{thm:matheron_finite} implies that we may enforce the condition $ \v{\tilde{f}}_{n} = \v{y}$ by writing\footnote{Practical variants of \eqref{eqn:pathwise_update_weights} avoid inverting $\m{\Phi}_{n}^{\vphantom{\top}}\m{\Phi}_{n}^\top$ by employing, e.g., Gaussian likelihoods (Section~\ref{sec:updates_gaussian}).} 
\[
\label{eqn:pathwise_update_weights}
\v{\phi}(\.)^{\top}(\v{w} \given \v{y})
\eqd
    \v{\phi}(\.)^{\top}
    \del{
        \v{w} + 
        \m{\Phi}_{n}^\top
        (\m{\Phi}_{n}^{\vphantom{\top}}\m{\Phi}_{n}^\top)^{-1}
        (\v{y} - \m{\Phi}_{n}\v{w})
    }.
\]
This result encourages us to approximate posteriors in much the same way as we have priors. After all, if we have chosen a basis $\v{\phi}$ that encodes our prior knowledge for $f$ (such as how smooth we believe this function to be), then it is reasonable to think that $\v{\phi}$ will further enable us to efficiently approximate $f \given \v{y}$. To the extent that this approach may seem like the natural evolution of ideas discussed in this section, we argue for the benefits of \emph{decoupling} the representation of the prior from that of the data.

The trouble with using a finite set of homogeneous basis functions $\v{\phi} = (\phi_{1}, \ldots, \phi_{\ell})$ to represent both the prior and the data is that these two tasks focus on different things. To accurately approximate a prior is to faithfully describe a random function $f$ on a domain $\c{X}$. Consequently, parsimonious approximations $\tilde{f}$ employ global basis functions that vary non-trivially everywhere on $\c{X}$.
This is largely why, e.g., Fourier features are an attractive choice for approximating stationary priors. But what of the data?

Conditioning on observations $\v{y}$ requires us to convey how our understanding of $f$ has changed. In most cases, we choose priors (and likelihoods) that reflect the belief that an observation $y_{i}$ only informs us about the process $f$ in the immediate vicinity of a point $\v{x}_{i}$. Updating $f$ to account for $\v{y}$, therefore, typically focuses on process values corresponding to specific regions of $\c{X}$. Rather than global basis functions, the data is best characterized by local ones that have near-zero values outside of the aforementioned regions. Not coincidentally, the canonical basis functions $k(\., \v{x})$ fit this description perfectly when the chosen prior implies that $y_i$ is only locally informative.

A key property of pathwise conditioning is that it not only provides us with a natural decomposition of GP posteriors---as sums of prior random variables and data-driven updates---but enables us to represent these terms in separate bases. Similar ideas can be found in recent works that explore alternative decompositions of Gaussian processes, such as separation of mean and covariance functions \cite{cheng17, salimbeni18} or decoupling of RKHS subspaces and their orthogonal complements \cite{shi19}. Unlike these works, however, we stress decoupling in the sense of using different classes of basis functions to represent different aspects of GP posteriors. While this type of decoupling is not unique to pathwise approaches \cite{lazaro2009inter, hensman17}, they drastically simplify the process by eliminating the need to analytically solve for sufficient statistics.

\begin{figure}
    \centering
    \includegraphics[width=\textwidth]{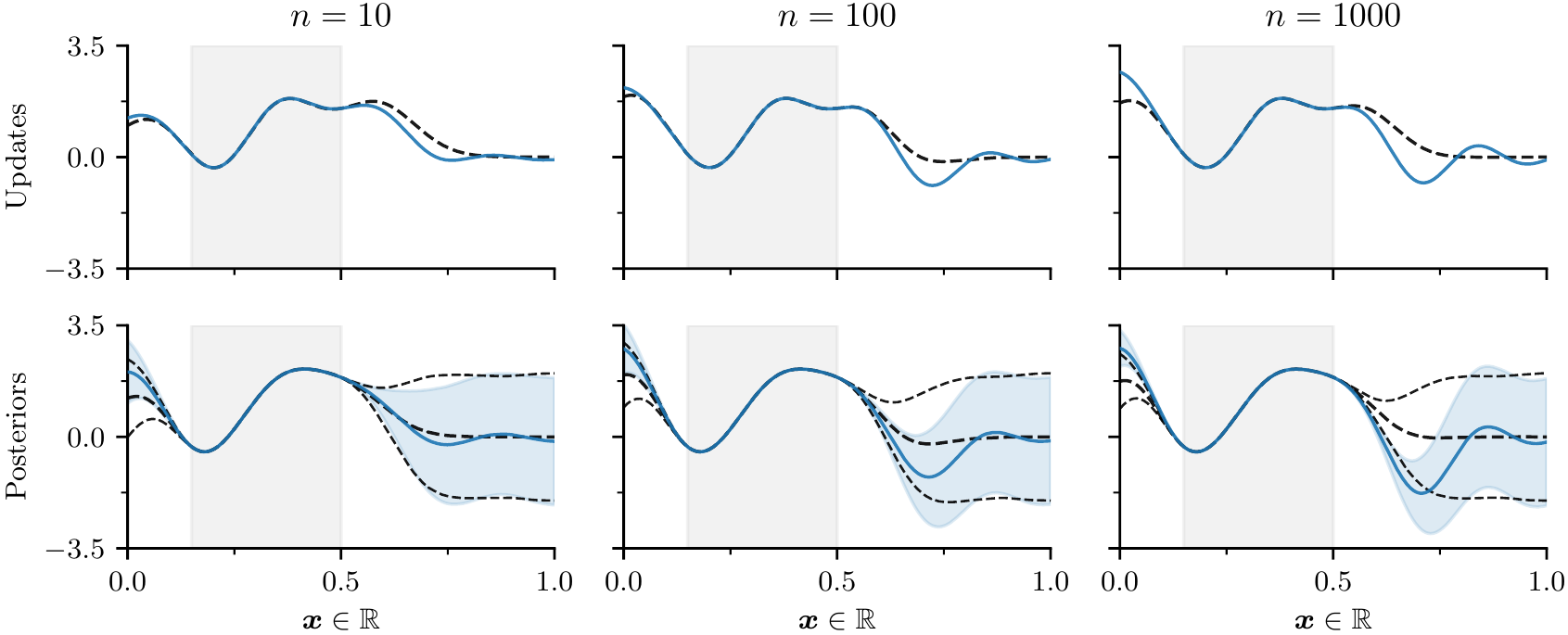}
    \caption{Overview of variance starvation when conditioning on $n \in \{10, 100, 1000\}$ observations of the form $y_{i} \sim \c{N}(f_i, 10^{-5})$ located within the gray shaded region. \emph{Top:} Comparison of pathwise updates to a single draw from an approximate prior $\tilde{f}(\.) = \v{\phi}(\.)^\top\v{w}$, constructed using $\ell = 1000$ Fourier features $\phi$. Updates defined using the same Fourier basis $\v{\phi}(\.)$ and the canonical basis functions $k(\., \m{X})$ are shown in blue and dashed-black, respectively. \emph{Bottom:} Mean and two standard deviations of the empirical posteriors formed by applying the aforementioned updates to $10^5$ draws from the approximate prior.}
    \label{fig:variance_starvation}
\end{figure}

This line of reasoning also helps to explain why finite-dimensional GPs constructed from homogeneous basis functions often produce poorly-calibrated posteriors. For now, we restrict our attention to the issue of \emph{variance starvation} \cite{wang18,mutny18,calandriello19} and return this topic in Section~\ref{sec:updates_discussion}. Figure~\ref{fig:variance_starvation} demonstrates what happens as the number of observations $n = \vert \v{y} \vert$ approaches the number of random Fourier features  $\ell = 1000$ used to approximate a squared exponential kernel. In general, the approximate posteriors produce extrapolations which become increasingly erratic. Note that the rate at which these defects materialize depends upon the choice of kernel and likelihood. In the figure, posteriors yielded by pathwise updates in canonical and Fourier bases (all other things being held equal) diverge as the number of observations $n$ approaches the number of random Fourier features $\ell$. This pattern emerges because the Fourier basis is better at describing stationary priors than non-stationary posteriors. Fourier features excel at capturing the global properties of the prior, but struggle to portray the localized effects of the data.

Of course, different types of data impose different kinds of conditions on the process $f$. We now examine various pathwise updates that enforce prominent types of conditions.

\section{Conditioning via pathwise updates}
\label{sec:updates}

Building off of the foundation prepared in Section~\ref{sec:gp_conditioning}, we now adapt Corollary~\ref{cor:matheron_inf} to accommodate different types of conditions and computational budgets. Throughout this section, we use $\v{\gamma}$ to denote the random variable realized by observations $\v{y}$ under the chosen likelihood.

\subsection{Gaussian updates}
\label{sec:updates_gaussian}
Corollary~\ref{cor:matheron_inf} treats observations $\v{y}$ as a realization of process values $\v{f}_{n} = f(\m{X}_{n})$. Hence, the conditions it imposes manifest as the equality constraint $\v{f}_{n} = \v{y}$. In the real world, however, we seldom observe $\v{f}_{n}$ directly. To account for this nuance, an observation $y$ is modeled by a \emph{likelihood} $p\del{y \given f(\v{x})}$. Viewed from this perspective, the equality constraint $\v{f}_{n} = \v{y}$ correspond to the limit where $p$ contracts to a point mass. Seeing as $y$ usually fails to fully disambiguate the true value of $f(\v{x})$, we typically employ likelihoods that induce weaker conditions than strict equalities.

For regression problems, the most common choice is to employ a Gaussian likelihood $p(y \given f(\v x)) = \c{N}(y \given f(\v x), \sigma^{2})$, the log of which penalizes the squared Euclidean distance of $f(\v{x})$ from $y$. Under the corresponding observation model $\v{\gamma} = \v{f}_{n} + \v{\varepsilon}$ with $\v{\varepsilon} \sim \c{N}(\v{0}, \sigma^{2} \m{I})$, $f$ and $\v{y}$ are jointly Gaussian. By Corollary~\ref{cor:matheron_inf} then, we may condition $f$ on $\v{\gamma} = \v{y}$ by writing
\[
\label{eqn:pathwise_update_gaussian}
\begin{split}
(f \given \v{\gamma} = \v{y})(\.)
&\eqd
    f(\.) 
    + k(\.,\m{X}) (\m{K}_{n,n} + \sigma^{2}\m{I})^{-1}(\v{y} - \v{f}_{n} - \v{\eps})
.
\end{split}
\]
Rather than exactly passing through observations $\v{y}$, the conditioned path $f \given \v{y}$ now smoothly interpolates between them.
In cases where $\v{\gamma}$ is not a Gaussian random variable, additional tools are needed.

\begin{figure}
    \centering
    \includegraphics[width=\textwidth]{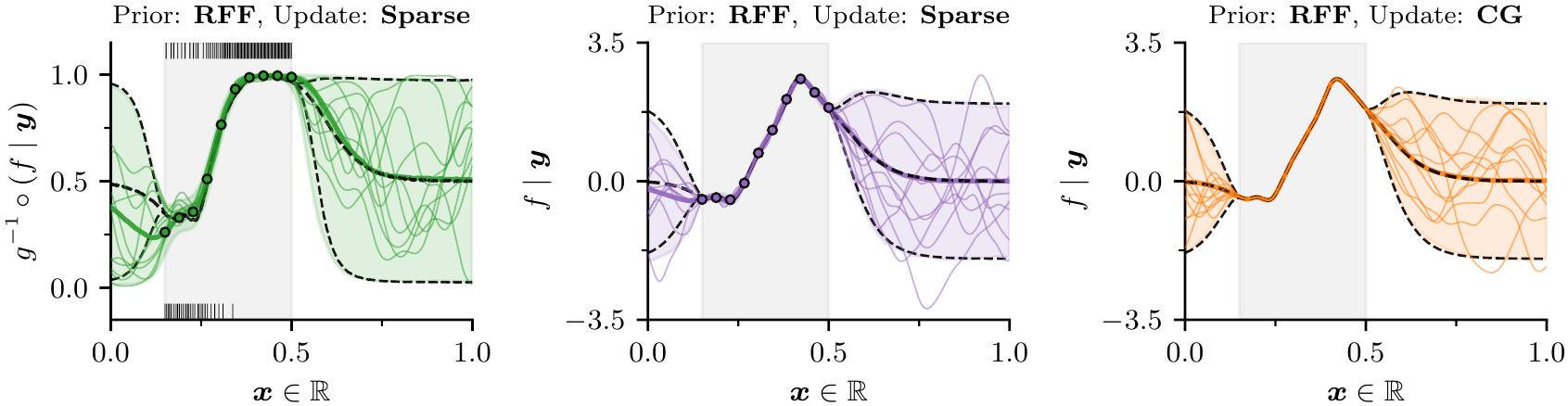}
    \caption{Visual comparison of different pathwise updates. \emph{Left and middle:} Variational inference is used to learn sparse updates at $m=10$ inducing locations $\m{Z}$ (circles). \emph{Right:} preconditioned conjugate gradients is used to iteratively solve for Gaussian updates. In all cases, 1000 observations $\v{y}$ are evenly spaced in the shaded region. Dashed lines denote mean and two standard deviations of ground truth posteriors, colored regions and thicker lines denote those of empirical ones. Middle and right plots illustrate regression with a Gaussian likelihood $\c{N}(y_{i} \given f_{i}, 10^{-3})$. The left plot shows binary classification with a Bernoulli likelihood and probit link function $g$; every tenth label is shown as a small, vertical bar.}
    \label{fig:update_rules}
\end{figure}

\subsection{Non-Gaussian updates}
\label{sec:updates_nongaussian}
In the general setting, where the random variable $\v{\gamma}$ is arbitrarily distributed under the chosen likelihood, $\v{\gamma}$ relates to process values $\v{f}$ by way of the non-conjugate prior
\[
\label{eqn:nonconjugate_prior}
p(\v{\gamma}, \v{f}) 
=
    p\left(\v{\gamma} \given g^{-1}(\v{f})\right)
    \c{N}(\v{f} \given \v{\mu}, \m{K}),
\]
where the \emph{link function} $g: \c{Y} \to \mathbb{R}$ maps from the space of predictions $\c{Y} \subset \R$ to the range of $f$. For binary classification problems, popular choices for $g: [0, 1] \to \mathbb{R}$ include logit and probit functions \cite{rasmussen06}. 
The left column of Figure~\ref{fig:update_rules} illustrates this scenario using methods described below.

Even under a non-conjugate prior \eqref{eqn:nonconjugate_prior}, the conditional expectation $\E(f \given \v{\gamma})$ and the residual $f - \E(f \given \v{\gamma})$ it induces are uncorrelated (see Section~\ref{sec:conditional_expectation}). Since $p(\v{f}, \v{\gamma})$ may not be Gaussian, however, it no longer follows that this lack of correlation implies independence---hence, the pathwise update \eqref{eqn:matheron_inf} may not hold.

Exact Bayesian inference and prediction are typically intractable when dealing with non-conjugate priors. Strategies for circumventing this issue generally approximate the true posterior by introducing an auxiliary random variable $\v{u} \sim q(\v{u})$ such that $f \given \v{u}$ resembles $f \given \v{y}$ according to a chosen measure of similarity \cite{nickisch2008approximations,hensman2015scalable}. 
For practical reasons, $\v{u}$ is typically assumed to be jointly Gaussian with $\v{f}$.\footnote{Note that, in the special case where $p(\v{f}, \v{\gamma})$ is Gaussian, the optimal $q$ is also Gaussian \cite{titsias09b}.} 
Consequently, non-conjugate priors $p(\v{f}, \v{\gamma})$ are replaced by conjugate ones $p(\v{f}, \v{u})$ to aid in the construction of approximate posteriors, whereupon Matheron's update rule holds once more. The following section explores these \emph{sparse} approximations in greater detail.

\subsection{Sparse updates}
\label{sec:updates_inducing}
Approximations to GP posteriors frequently revolve around conditioning a process $f$ on a random variable $\v{u} = (u_{1}, \ldots, u_{m}) \in \mathbb{R}^{m}$. Per the previous section, this may be because the outcome variable $\v{\gamma}$ is non-Gaussian \cite{nickisch2008approximations, titsias2010bayesian, hensman2015scalable}.
Alternatively, the $\c{O}(n^{3})$ cost for directly conditioning on all $n = \vert \v{y} \vert$ observations may be prohibitive \cite{titsias09a,hensman13}. In these cases and more, we would like to infer a distribution $q(\v{u})$ such that $f \given \v{u}$ explains the data. 
Defining (approximate) posteriors in this way not only avoids potential issues arising from non-Gaussianity of $\v{\gamma}$, but associates the computational cost of conditioning with $\v{u}$. As discussed below, this leads to pathwise updates that run in $\c{O}(m^{3})$ time.

Comprehensive treatment of different approaches to learning \emph{inducing distributions} $q(\v{u})$ is beyond the scope of this work. 
In general, however, these procedures operate by finding an approximate posterior $q(\v{f},\v{u})$ within a tractable family of approximating distributions $\c{Q}$.
For reasons that will soon become clear, this family of distributions typically includes an additional set of parameters $\m{Z}$, which help to define the joint distribution $p(\v{f}, \v{u})$. 
To help streamline presentation, we focus on the simplest and most widely used abstraction for inducing variables $\v{u}$: namely, \emph{pseudo-data}.

The noise-free pseudo-data framework \cite{snelson06, quinonero2007approximation, titsias09a} treats each draw of a random vector $\v{u} \sim q(\v{u})$ as a realization of process values $\v{f}_{m} = f(\m{Z})$ at a corresponding set of tunable locations $\m{Z} \in \c{X}^{m}$. This paradigm gets its name from the intuition that the (random) collection of pseudo-data $\del{\v{z}_{j}, u_{j}}_{j=1}^{m}$ mimics the effect of a noise-free data set $\del{\v{x}_{i}, f_{i}}_{i=1}^{n}$ on $f$. By construction, $\v{u}$ is jointly Gaussian with $f$.\footnote{This condition holds when $\v{u}$ relates to $f$ via a linear map \cite{lazaro2009inter}.} 
Appealing to Corollary~\ref{cor:matheron_inf}, we define the \emph{sparse pathwise update} as
\[
\label{eqn:pathwise_update_sparse}
(f \given \v{u})(\.) 
\eqd
    f(\.) +
    \underbracket[0.5pt]
    {\sum_{i=1}^m v_i k(\.,\v{z}_i)}_{\mathclap{m\t{-dimensional basis}}},
\]
where $\v{v} = \m{K}_{m,m}^{-1}\left(\v{u} - \v{f}_{m}\right)$. This formula is identical to the one given by Corollary~\ref{cor:matheron_inf}, save for the fact that we now sample $\v{u} \sim q(\v{u})$ and solve for a linear system involving the $m \times m$ covariance matrix $\m{K}_{m,m} = k(\m{Z}, \m{Z})$ at $\c{O}(m^{3})$ cost. The middle column of Figure~\ref{fig:update_rules} illustrates the sparse update induced by Gaussian $\v{u} \sim \c{N}(\v{\mu}_{\v{u}}, \m{\Sigma}_{\v{u}})$ with learned moments $\v{\mu}_{\v{u}}$ and $\m{\Sigma}_{\v{u}}$.

Just as we can imitate process values $\v{f}_{n}$, we can also emulate (Gaussian) observations $\v{y}$. This intuition leads to the Gaussian pseudo-data family of inducing distributions, whose moments
\[
\label{eqn:variational_family_pseudodata_gaussian}
\v{\mu}_{u} &= 
    \m{K}_{m,m}(\m{K}_{m,m} + \m{\Lambda})^{-1}\v{\tilde{y}}
&
\m{\Sigma}_{\v{u}} &= 
    (\m{K}_{m,m}^{-1} + \m{\Lambda})^{-1}
\]
are parameterized by \emph{pseudo-observations} $\v{\tilde{y}} \in \mathbb{R}^{m}$ and \emph{pseudo-noise} $\v{\tilde{\sigma}} \in \mathbb{R}_{+}^{m}$, where $\m{\Lambda} = \diag(\v{\tilde{\sigma}}^{2})$.
This choice of parameterization is motivated by the observation that, 
given $n \le m$ Gaussian random variables $\v{\gamma} \sim \c{N}(\v{f}_{n}, \sigma^{2}\m{I})$, the family of distributions it generates contains the optimal $q$ despite housing only $\c{O}(m)$ free terms \cite{seeger1999bayesian,opper2009variational}.\footnote{We recover the true posterior by, e.g., taking $(\tilde{y}_{i}, \tilde{\sigma}_{i}) = (y_{i}, \sigma)$ for all $i \le n$ and sending $\tilde{\sigma}_{i} \to \infty$ otherwise.} Using the Gaussian pathwise update \eqref{eqn:pathwise_update_gaussian}, we may express $\v{u}$ itself as
\[
\v{u} 
&\eqd 
    \v{f}_{m} + 
    \m{K}_{m,m}(\m{K}_{m,m} + \m{\Lambda})^{-1}(\v{\tilde{y}} - \v{f}_{m} - \v{\tilde{\eps}})
&
\v{\tilde{\eps}} & \sim \c{N}(\v{0}, \m{\Lambda}).
\]
Here, despite the fact that $\v{f}_{m}$ and $\v{\tilde{\eps}}$ generate $\v{u}$, it remains the case that $\Cov(\v{f}_{m} + \v{\tilde{\eps}}, \v{u}) = \m{0}$.
Substituting this expression into \eqref{eqn:pathwise_update_sparse} and simplifying gives the pathwise update\footnote{This same line of reasoning leads to a \emph{rank-1 pathwise update} for cases where conditions arrive online.}
\[
\label{eqn:pathwise_update_pseudodata_gaussian}
(f \given \v{u})(\.)
\eqd
    f(\.) +
    k(\., \m{Z})
    \del{\m{K}_{m,m}+  \m{\Lambda}}^{-1}
    \del{\v{\tilde{y}} - \v{f}_{m} - \v{\tilde{\eps}}}.
\]
Hence, while sampling $\v{u}$ is more complicated in the Gaussian pseudo-data case, the resulting pathwise update is straightforward. This family of inducing distributions is particularly advantageous in the large $m$ setting, both because it contains only $\c{O}(m)$ free parameters and for reasons discussed in the following section.

In rough analogy to methods discussed in Section~\ref{sec:priors}, we may think of the sparse updates introduced here as using an $m$-dimensional basis $k(\., \m{Z})$ to approximate functions defined in terms of the $n$-dimensional basis $k(\., \m{X}_{n})$. In practice, this basis is often efficient because neighboring training locations give rise to similar basis functions. Kernel basis functions at appropriately chosen sets of $m \ll n$ locations $\m{Z}$ exploit this redundancy to produce a sparser, more cost-efficient representation. \textcite{burt20} study this problem in detail and derive bounds on the quality of variational approximations to GP posteriors as $m \to n$.

\subsection{Iterative solvers}
\label{sec:updates_solvers}
Throughout this section, we have focused on the high-level properties of pathwise updates in relation to various problem settings. We have said little, however, regarding the explicit means of executing such an update. In all cases discussed here, pathwise updates have amounted to solutions to system of linear equations. For example, the update originally featured in Corollary~\ref{cor:matheron_inf} solves the system $\m{K}_{n,n}\v{v} = \v{y} - \v{f}_{n}$ for a vector of coefficients $\v{v}$, which define how the same realization of $f$ changes when subjected to the condition $\v{f}_{n} = \v{y}$. Given a reasonable number of conditions $n$ (up to several thousand), we may obtain $\v{v}$ by first computing the Cholesky factor $\m{L}_{n,n} = \m{K}^{\nicefrac{1}{2}}_{n,n}$ and then solving for a pair of triangular systems $\m{L}_{n,n}\v{\bar{v}} = \v{u} - \v{f}_n$ and $\m{L}_{n,n}^{\top} \v{v} = \v{\bar{v}}$. For large $n$, however, the $\c{O}(n^{3})$ time complexity for carrying out this recipe is typically prohibitive. 

Rather than solving for coefficients $\v{v}$ directly, we may instead employ an \emph{iterative solver} that constructs a sequence of estimates $\v{v}^{(1)},\v{v}^{(2)},\dotsc$ to $\v{v}$, such that $\v{v}^{(j)}$ converges to the true $\v{v}$ as $j$ increases.
Depending on the numerical properties of the linear system in question, it is possible (or even likely) that a high-quality estimate $\v{v}^{(j)}$ will be obtained after only $j \ll n$ iterations. This line of reasoning features prominently in a number of recent works, where iterative solvers have been shown to be highly competitive for purposes of approximating GP posteriors \cite{pleiss2018constant,gardner2018gpytorch,wang2019exact}. The right column of Figure~\ref{fig:update_rules} visualizes an iterative solution to the Gaussian pathwise update \eqref{eqn:pathwise_update_gaussian} obtained using preconditioned conjugate gradients \cite{gardner2018gpytorch}.

In these cases, posterior sampling via pathwise conditioning enjoys an important advantage over distributional approaches: it allows us to solve for linear system of the form $\m{K}^{-1}_{n,n}\v{v}$ rather than working with $\m{K}^{\nicefrac{1}{2}}_{*,*\given n} \v{\zeta}$.
Whereas the former amounts to a standard solve, the latter often requires special considerations \cite{pleiss2020fast} and can be difficult to work with when typical square root decompositions prove impractical \cite{parker12}.

Lastly, we note that these techniques can be combined with sparse approximations for improved scaling in $m$ and faster convergence of iterative solves. As a concrete example, we return to the Gaussian pseudo-data variational family \eqref{eqn:variational_family_pseudodata_gaussian}. By construction, the corresponding pathwise update \eqref{eqn:pathwise_update_pseudodata_gaussian} closely resembles the original Gaussian update \eqref{eqn:pathwise_update_gaussian}. In general, however, pseudo-noise variances $\tilde{\sigma}_{i}^{2}$ are often significantly larger than the true noise variance $\sigma^{2}$. The resulting linear system $(\m{K}_{m,m} + \m{\Lambda})^{-1}\v{v}$ is, therefore, substantially better-conditioned than that of the exact alternative---implying that it can be solved in far fewer iterations.

\subsection{Discussion}
\label{sec:updates_discussion}

In Section~\ref{sec:priors_discussion}, we discussed finite-dimensional approximations of Gaussian process posteriors. There, we explored how the globality of the prior reinforces the use of basis functions $\phi_i: \c{X} \to \R$ that inform us about $f$ over the entire domain $\c{X}$, while the localized effects of the data encourages the use of $\phi_i$ that only tell us about $f$ on subsets of $\c{X}$. This conflict hinders our ability to efficiently represent both the prior and the data (i.e., the posterior) using a single class of basis functions. That discussion ended with a demonstration of what happens when $\v{\phi} = (\phi_{1}, \ldots, \phi_{\ell})$ solely consists of global basis functions, specifically random Fourier features. Most works, however, have focused on the use of canonical basis functions $k(\., \v{x})$, which are typically local. This section, therefore, aims to fill in the gaps.

At the end of Section~\ref{sec:priors_discussion}, we saw how trouble conveying the data in global bases led to approximate posteriors that were starved for variance (Figure~\ref{fig:variance_starvation}). Writing the update rules---for a draw from an approximate prior $\tilde{f}(\.) = \v{\phi}(\.)^\top\v{w}$ subject to the condition $\v{\tilde{f}}_{n} = \v{y}$---in both unified and decoupled bases side-by-side helps to highlight their key differences
\[
\label{eqn:approximate_posteriors_comparison}
&\underbracket[0.5pt]{
    \tilde{f}(\.) +
    \v{\phi}(\.)^\top \m{\Phi}_{n}^\top \del[1]{\m{\Phi}_{n}^{\vphantom{\top}}\m{\Phi}_{n}^\top}^{-1}
    \del[1]{\v{y} - \v{\tilde{f}}_{n}}
    \vphantom{\tilde{f}(\.) + 
    k(\., \m{X}_{n})\m{K}_{n,n}^{-1} (\v{y} - \v{\tilde{f}}_{n})}
    }_{\t{unified approximate posterior}}
&\underbracket[0.5pt]{
    \tilde{f}(\.) + 
    k(\., \m{X}_{n})\m{K}_{n,n}^{-1}
    \del[1]{\v{y} - \v{\tilde{f}}_{n}}
    }_{\t{decoupled approximate posterior}}.
\]
On the right, the cross-covariance term $\v{\phi}(\.)^\top \m{\Phi}_{n}^\top = \v{\phi}(\.)^\top \v{\phi}(\m{X}_{n})$ is replaced by $k(\., \m{X}_{n})$. Seeing as the former is often chosen to approximate the latter in a way that converges when an appropriate limit is taken, for instance in \eqref{eqn:rkhs_inner_prod_onb}, it comes as no surprise that $k(\., \m{X}_{n})$ more accurately represents data. Moreover, the matrix inverse $\del[1]{\m{\Phi}_{n}^{\vphantom{\top}}\m{\Phi}_{n}^\top}^{-1}$ appearing on the left is often ill-conditioned and, therefore, amplifies numerical errors. Finite-dimensional GPs constructed from local basis functions exhibit similar issues, albeit for essentially the opposite reason. Rather than failing to adequately represent the data, local basis functions struggle to reproduce the prior.

Many approaches to approximating Gaussian processes $f \sim \c{GP}(0, k)$ revolve around representing the data in terms of $m$-dimensional canonical bases $k(\., \m{Z})$; for a review, see \textcite{quinonero2007approximation}. Early iterations of this strategy \cite{silverman1985some, wahba1990spline, tipping2000relevance}, typically used $k(\., \m{Z})$ to define degenerate Gaussian processes \cite{rasmussen06}. Here, the term \emph{degenerate} emphasizes the fact that the covariance function 
\[
\label{eqn:degnerate_kernel}
\tilde{k}(\v{x}_{i}, \v{x}_{j})
=
    k(\v{x}_{i}, \m{Z})k(\m{Z},\m{Z})^{-1}k(\m{Z}, \v{x}_{j})
\]
of such a process has a finite number of non-zero eigenvalues. From the weight-space perspective, degenerate GPs are Bayesian linear models $\tilde{f}(\.) = k(\., \m{Z}) \v{w}$, which makes it clear that $\tilde{f}(\.)$ goes to zero as $k(\., \m{Z}) \to \v{0}$. This behavior is particularly troublesome if all $\v{z} \in \m{Z}$ are positioned near training locations $\m{X}_{n}$: since $k(\v{x}_{*}, \m{Z})$ typically vanishes as $\v{x}_{*}$ retreats from $\m{Z}$, both the prior and the posterior collapse to point masses away from the data.

Instead of focusing on the data, one idea is to start by finding a basis $k(\., \m{Z})$ capable of accurately reproducing the prior. Accomplishing this feat will require us to use a relatively large number of basis functions, since $\m{Z}$ will need to effectively cover the (compact) domain $\c{X}$. As mentioned in Section~\ref{sec:priors_exact}, certain kernels produce special kinds of matrices when evaluated on particular sets. Exploiting these special properties---e.g., by taking the Toeplitz matrices formed when evaluating a stationary product kernel $k$ on a regularly spaced grid $\m{Z}$ and embedding them inside of circulant ones \cite{wood1994simulation, dietrich97}---enables us to drastically reduce the cost of expensive matrix operations, such as multiplies, decompositions, and inverses.
Especially when $\c{X}$ is low dimensional, then, we can use the canonical basis to efficiently approximate the prior.

Kernel interpolation methods \cite{wilson2015kernel, pleiss2018constant} take this idea a step further. Given a set of $m$ inducing locations $\m{Z}$, let $\v{\xi}: \c{X} \to \R^{m}$ be a \emph{weight function} \cite{silverman1984spline} mapping locations $\v{x}_{i}$ onto (sparse) weight vectors $\v{\xi}_{i}$ such that $k(\v{x}_{i}, \m{Z}) \approx \v{\xi}_{i}^\top k(\m{Z}, \m{Z})$. By applying this technique to \eqref{eqn:degnerate_kernel}, we can define another Gaussian process $g \sim \c{GP}(0, c)$ with degenerate covariance $c(\v{x}_{i}, \v{x}_{j}) = \v{\xi}_{i}^\top k(\m{Z},\m{Z}) \v{\xi}_{j}^{\vphantom{\top}}$. As a Bayesian linear model, we have $g(\.) = \v{\xi}(\.)^\top \v{g}_{m}$. Notice that process values $\v{g}_{m} = g(\m{Z})$ now play the role of random weights $\v{w}$ and fully determine the behavior of the random function $g$. Assuming $\m{Z}$ was chosen so that $k(\m{Z}, \m{Z})$ admits convenient structure, random vectors $\v{g}_{m} \given \v{y}$ and, hence, random functions $(g \given \v{y})(\.)$ can be obtained cheaply \cite{pleiss2018constant}. When $\m{Z}$ is sufficiently dense in $\c{X}$ (so as to be reasonably close to $\v{x}_{*}$), this strategy provides an alternative means of efficiently sampling from GP posteriors.

\subsection{An empirical study}
\label{sec:updates_empiricalStudy}

\begin{figure}[t!]
    \centering
    \includegraphics[width=\textwidth]{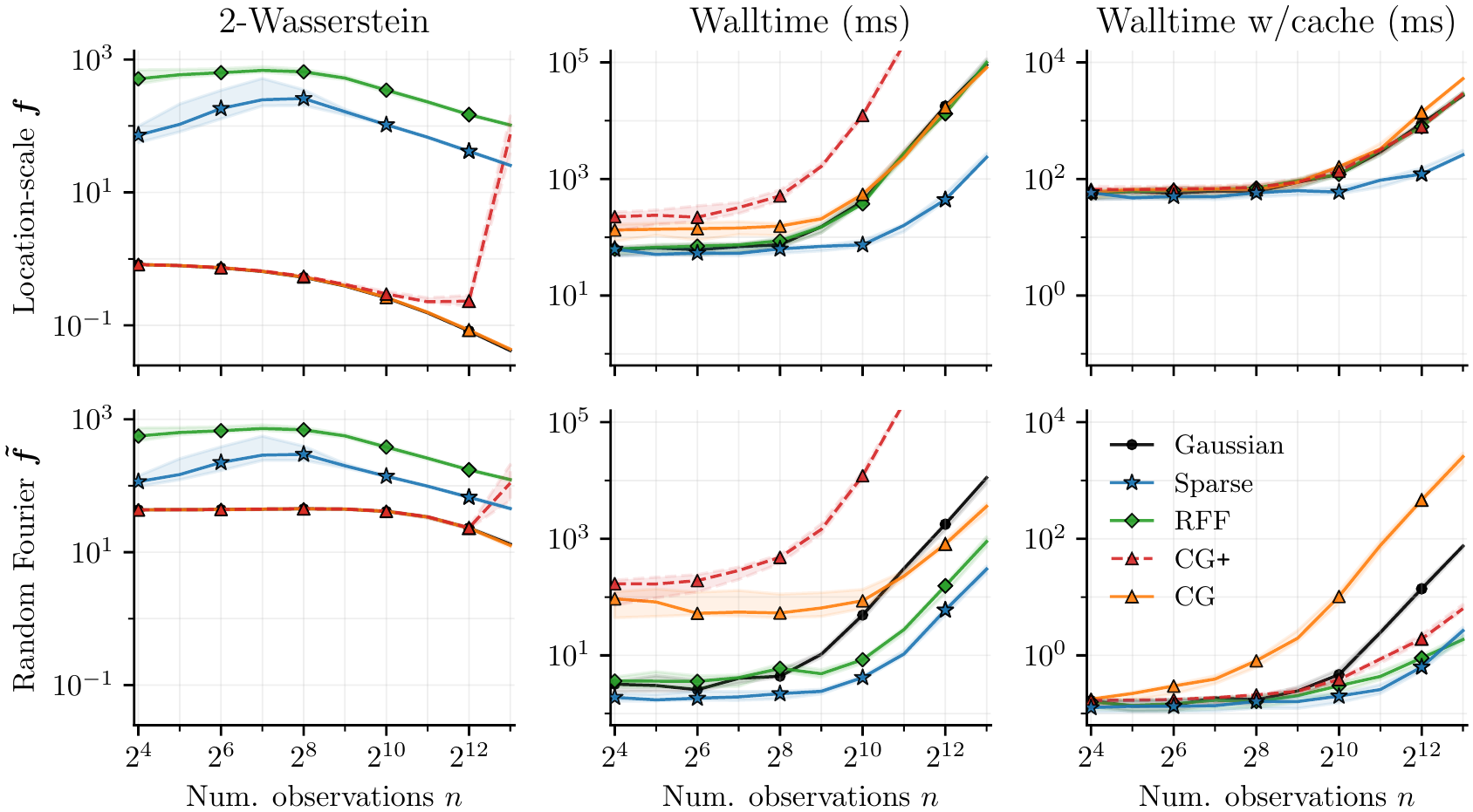}
    \caption{
        Accuracy and cost of different methods for sampling from GP posteriors given $n$ observations. Draws from the prior are generated using either location-scale (\emph{top}) or $\ell = 4096$ random Fourier features (\emph{bottom}). We denote Gaussian updates by black dots, sparse updates by blue stars, CG updates by orange and red triangles, and RFF updates by green diamonds.
        Sparse and RFF updates both utilized $m = \frac{n}{4}$ basis functions. All results are reported as medians and interquartile ranges measured over 32 independent trials.
        \emph{Left:} 2-Wasserstein distances of empirical distributions of $10^{5}$ samples from the ground truth GP posterior. \emph{Middle and right:} Time taken to generate a draw of $(\v{f}_{*} \given \.\,) \in \mathbb{R}^{1024}$ with and without caching of terms that are independent of $\m{X}_{*}$.
    }
    \label{fig:accuracy_vs_cost}
\end{figure}

By now, we have explored a variety of techniques for sampling from GP posteriors. Each of these methods is well suited for a particular type of problem. To help shed light on their respective niches, we conducted a simple controlled experiment. 

Here, our goal is to better understand how different methods balance the tradeoff of cost and accuracy.
We measured cost in terms of runtimes and accuracy in terms of 2-Wasserstein distances between empirical distributions and true posterior (see Section~\ref{sec:error}). 
To eliminate confounding variables, we assumed a known Mat\'{e}rn-$\nicefrac{5}{2}$ prior on random functions $f : \mathbb{R}^{4} \to \mathbb{R}$. All trials began by sampling this prior at $n$ training locations $\m{X}_{n}$ and 1024 test locations $\m{X}_{*}$, using either location-scale transforms or random Fourier features. We then used the various update rules explored in this section to condition on $n$ observations $\v{y} \sim \c{N}\del{\v{f}_{n}, 10^{-3} \m{I}}$.

Sparse updates were constructed using $m = \frac{n}{4}$ inducing variables $\v{u}$, whose distributions $q(\v{u})$ and inducing locations $\m{Z}$ were obtained by minimizing Kullback--Leibler divergences. Conjugate-gradient-based updates were carried out by, first, computing partial pivoted Cholesky decompositions in order to precondition linear systems $(\m{K}_{n,n} + \sigma^{2}\m{I})\v{v} = (\v{y} - \v{f}_{n} - \v{\varepsilon})$. We then iteratively solved for Gaussian pathwise updates using the method of conjugate gradients. Stopping conditions for both the partial pivoted Cholesky decomposition and conjugate gradient solver were chosen to match those of \textcite{gardner2018gpytorch}. Prior to discussing trends in Figure~\ref{fig:accuracy_vs_cost}, we would like to point out that curves associated with Gaussian updates (black) are heavily obscured: in the left column, by CG-based ones (orange and red) and in top middle and top right plots by RFF-based ones (green).

Comparing the rows of Figure~\ref{fig:accuracy_vs_cost}, we see that random Fourier feature (RFF) approximations to priors introduce modest amounts of error in exchange for large cost reductions. These savings are particularly dramatic in cases where test inputs $\m{X}_{*}$ significantly outnumber training locations $\m{X}_{n}$. Echoing discussion in Section~\ref{sec:priors_discussion}, however, $m$-dimensional random Fourier bases struggle to represent the data. All other things being held equal, sparse updates performed in the canonical basis consistently outperform RFF-based ones. These sparse methods are also considerable faster than competing approaches when $m \ll n$.

Direct comparison of sparse and CG updates is difficult, since both methods are sensitive to various design choices. In our experiments, CG-based updates behaved tantamount to exact ones---with two important caveats. First, CG-based updates were initially slower than exact ones but outpace them as $n$ increased. Second, na\"ively computing pathwise updates using CG is highly inefficient when it comes to caching. When repeatedly conditioning on (potentially different realizations of) $\v{\gamma} = \v{y}$, one option is to use CG to precompute the matrix inverse $(\m{K}_{n,n} + \sigma^{2}\m{I})^{-1}$. This CG$\texttt{+}$ variant is significantly more cache-friendly, but also much more susceptible to round-off error---see dashed red curves in Figure~\ref{fig:accuracy_vs_cost}.

These empirical results help to characterize the behaviors of errors introduced by different approximation schemes, but leave many questions unanswered. In order to fill in some of the remaining gaps, we now analyze various types of approximation error in details.

\section{Error analysis}
\label{sec:error}

Over the course of this section, we will analyze the different types of error introduced by pathwise approximations. Speaking about these errors requires us to agree upon a suitable notion of similarity between Gaussian processes. Ultimately, we are interested in understanding how these approximations influence Monte Carlo estimators. We therefore focus on $2$-Wasserstein distances between true and approximate posteriors, since they control downstream Monte Carlo errors.\footnote{$2$-Wasserstein distances majorize $1$-Wasserstein distances, which regulate expectations of Lipschitz functionals by Kantorovich--Rubinstein duality \cite{villani08}.} These distances measure the similarity of Gaussian processes $\tilde{f}$ and $f$ as the expectation of a metric $d\del[1]{\tilde{f}, f}$ under the best possible \emph{coupling} of the two processes. Formally, we have 
\[
W_{2,d}\del[1]{\tilde{f}, f}
=
    \sbr{\inf_{\pi \in \Pi(\tilde{\mu}, \mu)}
    \E_{\pi}d\del[1]{\tilde{f}, f}^2}^{\nicefrac{1}{2}},
\]
where $\Pi(\tilde{\mu}, \mu)$ denotes the set of valid couplings \cite{mallasto2017learning}, i.e. joint measures whose marginals correspond with the Gaussian measures $\tilde{\mu}$ and $\mu$ induced by processes $\tilde{f}$ and $f$, respectively.
Below, we employ $L^2$ and supremum norms as the underlying metrics used to define $2$-Wasserstein distances.

\newcommand{\infLowerSubscript}{\mathop{\mathrm{inf}\vphantom{\mathrm{sup}}}}
For the remainder of this section, we assume that the domain $\c{X}$ is a compact subset of some metric measure space $\c{M}$ and that $\c{X}$ has finite measure.
As a straightforward example, the domain may be a $d$-dimensional hypercube $\c{X} = [a, b]^{d}$ within $\c{M} = \R^d$.

Lastly, let us introduce some additional notation to simplify material presented below. First, we will use $\tilde{f} \given \v{y}$ and $\tilde{f} \given \v{u}$ to denote pathwise conditioning of an approximate prior $\tilde{f}$ via canonical \eqref{eqn:matheron_inf} and sparse \eqref{eqn:pathwise_update_sparse} update rules, respectively. These constructions should not be confused with the approximate posteriors discussed in Sections~\ref{sec:priors_discussion} and \ref{sec:updates_discussion}. Second, we will superscript covariance functions $k$ to convey their corresponding processes. For example, $k^{(\tilde{f})}$ will denote the kernel of the approximation prior $\tilde{f}$. Third and finally, given a set of $n$ training locations $\m{X}_{n} \subset \c{X}$, define the \emph{weight function} $\v{\xi} : \c{X} \to \R^{n}$ as
\[
\v{\xi}(\.) = k(\m{X}_{n}, \m{X}_{n})^{-1} k(\m{X}_n,\.)
.
\]
Variants of this function have been extensively studied in the context of regression; see \textcite{silverman1984spline, sollich2005using} and references contained therein.

\subsection{Posterior approximation errors}
\label{sec:post-error}
This section adapts the results of \textcite{wilson20} to study the error in the decoupled approximate posterior
\[
\label{eqn:gp_posterior_pathwise_decoupled}
(\tilde{f} \given \v{y})(\.) 
\eqd 
    \tilde{f}(\.) + k(\., \m{X}_{n})\m{K}_{n,n}^{-1}(\v{y} - \v{\tilde{f}}) 
= 
    \tilde{f}(\.) + \v{\xi}(\.)^\top (\v{y} - \v{\tilde{f}}) 
\]
formed by updating an $\ell$-dimensional approximate priors $\tilde{f}(\.) = \v{\phi}(\.)^\top\v{w}$ via an $n$-dimensional canonical basis $k(\., \m{X}_{n})$ so as to satisfy the condition imposed by $n$ noise-free observations $\v{y}$.

\begin{proposition}
\label{prop:wasserstein_bound}
Assume that $\c{X} \subset \R^d$ is compact and that the stationary kernel $k$ is sufficiently regular for $f \sim \c{GP}(\mu, k)$ to be almost surely continuous. 
Accordingly, if we define 
$C_{1} 
= 
    \sqrt{2}
    \f{diam}(\c{X})^{\nicefrac{d}{2}}
    \del[1]{1 + 
            \norm[0]{k}_{\c{C}(\c{X}^2)}^2
            \norm[0]{\m{K}_{n,n}^{-1}}_{L(\ell^\infty; \ell^1)}^2
    }^{\nicefrac{1}{2}},$
then we have
\[
\label{eqn:wasserstein_bound}
W_{2,L^2(\c{X})} \del[1]{\tilde{f} \given \v{y}, f \given \v{y}}
=
    \del{
        \inf_{\pi \in \Pi(\tilde{\mu}, \mu)}
        \E_{\pi}\norm{(\tilde{f} \given \v{y}) - (f \given \v{y})}_{L^2(\c{X})}^2
    }^{\nicefrac{1}{2}}
\leq 
    C_{1} W_{2, \c{C}(\c{X})}\del[1]{\tilde{f}, f},
\]
where $W_{2,L^2(\c{X})}$ and $W_{2,\c{C}(\c{X})}$ respectively denote 2-Wasserstein distances over the Lebesgue space $L^2(\c{X})$ and the space of continuous functions $\c{C}(\c{X})$ equipped with the supremum norm, $\norm{\.}_{\c{C}(\c{X}^2)}$ is the supremum norm over continuous functions, and $\norm[0]{\.}_{L(\ell^\infty; \ell^1)}$ is the operator norm between $\ell^\infty$ and $\ell^1$ spaces.
\end{proposition}

\begin{proof} We begin by considering the term inside the expectation in \eqref{eqn:wasserstein_bound}. Applying Matheron's rule followed by H\"older's inequality ($p=1$, $q=\infty$), we have
\[
\begin{split}
\label{eqn:prop5_proof_part1}
\abs{(\tilde{f} \given \v{y})(\v{x}) - (f \given \v{y})(\v{x})}^2 
&\leq 
    2\abs{\tilde{f}(\v{x}) - f(\v{x})}^2 
    + 2\abs{\v{\xi}(\v{x})^\top(\v{\tilde{f}}_{n} - \v{f}_{n})}^2
\\
&\leq 
    2\norm{\tilde{f} - f}_{L^\infty(\c{X})}^2
    + 2\norm{\v{\xi}(\v{x})}_{\ell^1}^2
    \norm{\v{\tilde{f}}_{n} - \v{f}_{n}}_{\ell^\infty}^2.
\end{split}
\]
Continuing from the second line, the definition of the operator norm implies that
\[
\label{eqn:prop5_proof_part2}
\begin{split}
\abs{(\tilde{f} \given \v{y})(\v{x}) - (f \given \v{y})(\v{x})}^2
&\leq 
    2\del{1 + 
        \norm{k(\v{x}, \m{X}_{n})}_{\ell^\infty}^{2}
        \norm{\m{K}_{n,n}^{-1}}_{L(\ell^\infty; \ell^1)}^2
    } \norm{\tilde{f} - f}_{L^\infty(\c{X})}^2
\\
&\leq 
    2\del{1 + 
        \norm[1]{k}_{\c{C}(\c{X}^2)}^2 
        \norm{\m{K}_{n, n}^{-1}}_{L(\ell^\infty; \ell^1)}^2
    } \norm{\tilde{f} - f}_{L^\infty(\c{X})}^2
\\
&=
    \underbracket[0.5pt]{2\del{1 + 
        \norm[1]{k}_{\c{C}(\c{X}^2)}^2
        \norm{\m{K}_{n, n}^{-1}}_{L(\ell^\infty; \ell^1)}^2
    }}_{C_{0}} \norm{\tilde{f} - f}_{\c{C}(\c{X})}^2,
\end{split}
\]
where, in the final line, we have used continuity of sample paths to replace  $\norm{\cdot}_{L^\infty(\c{X})}$ with $\norm{\cdot}_{\c{C}(\c{X})}$.
We now lift this bound between sample paths to one on 2-Wasserstein distances by integrating both sides with respect to the optimal coupling $\pi \in \Pi(\tilde{\mu}, \mu)$
\[
\begin{split}
W_{2,L^2(\c{X})}\del[1]{\tilde{f} \given \v{y}, f \given \v{y}}
&=
    \del{\inf_{\pi \in \Pi(\tilde{\mu}, \mu)}
    \E_{\pi}\norm{(\tilde{f} \given \v{y}) - (f \given \v{y})}_{L^2(\c{X})}^2 }^{\nicefrac{1}{2}}
\\
&\leq
    \del{C_{0} \f{vol}(\c{X})
    \inf_{\pi \in \Pi(\tilde{\mu}, \mu)}
    \E_{\pi}\norm{\tilde{f} - f}_{\c{C}(\c{X})}^2}^{\nicefrac{1}{2}}
\\
&\le
    C_{1} W_{2, \c{C}(\c{X})}\del[1]{\tilde{f}, f},
\end{split}
\]
where $\f{vol}(\c{X})$ denotes the Lebesgue measure of $\c{X}$. Hence, the claim follows.\footnote{Note that, since $f$ is sample-continuous and $\c{C}(\c{X})$ is a separable metric space, $W_{2,\c{C}(\c{X})}$ is a proper metric.}
\end{proof}

\begin{proposition}
\label{prop:covariance_error}
With the same assumptions, let $C_{2} = n \del{1 + \norm[0]{\m{K}^{-1}_{n,n}}_{\c{C}(\c{X}^2)}\norm[0]{k}_{\c{C}(\c{X}^2)}}^2$. Then,
\[
\label{eqn:covariance_error}
\mathbb{E}_{\v{\phi}}
\norm[1]{k^{(\tilde{f} \given \v{y})} - k^{(f \given \v{y})}}_{\c{C}(\c{X}^2)} \leq C_{2} \mathbb{E}_{\v{\phi}} \norm[1]{k^{(\tilde{f})} - k}_{\c{C}(\c{X}^2)}
.
\]
Moreover, when $\tilde{f}$ is a random Fourier feature approximation of the prior, it follows that
\[
    \mathbb{E}_{\v{\phi}}
    \norm[1]{k^{(\tilde{f} \given \v{y})} - k^{(f \given \v{y})}}_{\c{C}(\c{X}^2)} \leq \ell^{\nicefrac{-1}{2}}C_{2} C_{3},
\]
where $C_{3}$ is one of several possible constants given by \textcite{sutherland15}.
\end{proposition}

\begin{proof}
Let $M_k: \c{C}(\c{X}\times\c{X})\to \c{C}(\c{X}\times\c{X})$ be the bounded linear operator given by
\[
(M_k c)(\v{x}, \v{x}') 
&= 
    c(\v{x}, \v{x}') 
    - c(\v{x}, \m{X}_{n}) \v{\xi}(\v{x}')
    - \v{\xi}(\v{x})^\top c(\m{X}_{n}, \v{x}') 
    + \v{\xi}(\v{x})^\top c(\m{X}_{n}, \m{X}_{n}) \v{\xi}(\v{x}').
\]
Henceforth, we omit the subscript from $M_{k}$ to ease notation. Note that, by construction,
\[
k^{(f \given \v{y})}(\v{x}, \v{x}')
&= 
    (M k)(\v{x}, \v{x}')
&
k^{(\tilde{f} \given \v{y})}(\v{x}, \v{x}')
&= 
    (M k^{(\tilde{f})})(\v{x}, \v{x}').
\]
Focusing on the integrand on the left-hand side of \eqref{eqn:covariance_error}, we begin by separating out the operator norm $\norm{M}_{L(\c{C}(\c{X}^2);\c{C}(\c{X}^2))}$ as
\[
\norm{k^{(\tilde{f} \given \v{y})} - k^{(f \given \v{y})}}_{\c{C}(\c{X}^2)}
=
    \norm{M k^{(\tilde{f})} - M k}_{\c{C}(\c{X}^2)}
\le
    \norm{M}_{L(\c{C}(\c{X}^2);\c{C}(\c{X}^2))}
    \norm{k^{(\tilde{f})} - k}_{\c{C}(\c{X}^2)}.
\]
Refining this inequality requires us to upper bound $\norm{M}_{L(\c{C}(\c{X}^2);\c{C}(\c{X}^2))}$. To do so, we write
\[
\label{eqn:M_kc}
\norm{M c}_{\c{C}(\c{X}^2)} 
&\leq 
    \norm{c}_{\c{C}(\c{X}^2)}
    + 2 \norm{c(\., \m{X}_{n})\v{\xi}(\.)}_{\c{C}(\c{X}^2)}
    + \norm{\v{\xi}(\.)^\top c(\m{X}_{n}, \m{X}_{n}) \v{\xi}(\.)}_{\c{C}(\c{X}^2)}
.
\]
We now use H\"{o}lder's inequality ($p = 1$, $q = \infty$) followed by the definition of the operator norm $\norm{\cdot}_{L(\ell^\infty; \ell^1)}$ to bound the second and third terms on the right as
\[
\begin{split}
\norm{c(\., \m{X}_{n})\v{\xi}(\.)}_{\c{C}(\c{X}^2)}
&= 
    \sup_{\v{x}, \v{x}' \in \c{X}}
    \sbr{c(\v{x}, \m{X}_{n})\v{\xi}(\v{x}')}
\\
&\leq 
    \sup_{\v{x}, \v{x}' \in \c{X}} 
    \sbr{
        \norm{c(\v{x}, \m{X}_{n})}_{\ell^\infty}
        \norm{\m{K}_{n,n}^{-1}}_{L(\ell^\infty; \ell^1)}
        \norm{k(\m{X}_{n}, \v{x}')}_{\ell^\infty}
    }
\\
&\leq 
    \norm{c}_{\c{C}(\c{X}^2)}
    \norm{\m{K}^{-1}_{n,n}}_{L(\ell^\infty; \ell^1)}
    \norm{k}_{\c{C}(\c{X}^2)}
\end{split}
\]
and
\[
\norm{\v{\xi}(\.)^\top c(\m{X}_{n}, \m{X}_{n}) \v{\xi}(\.)}_{\c{C}(\c{X}^2)}
\leq 
    n
    \norm{c}_{\c{C}(\c{X}^2)}
    \norm{\m{K}^{-1}_{n,n}}_{L(\ell^\infty; \ell^1)}^2
    \norm{k}_{\c{C}(\c{X}^2)}^2.
\]
Returning to \eqref{eqn:M_kc}, we may now bound $\norm{M c}_{\c{C}(\c{X}^2)}$ by writing
\[
\begin{split}
\norm{M c}_{\c{C}(\c{X}^2)}
&\leq
    \norm{c}_{\c{C}(\c{X}^2)}
    \del[2]{1 
        + 2 \norm{\m{K}^{-1}_{n,n}}_{L(\ell^\infty; \ell^1)}\norm{k}_{\c{C}(\c{X}^2)}
        + n \norm{\m{K}^{-1}_{n,n}}_{L(\ell^\infty; \ell^1)}^2 \norm{k}_{\c{C}(\c{X}^2)}^2
    }
\\
&\leq 
    \norm{c}_{\c{C}(\c{X}^2)}
    \del[2]{
        n 
        \sbr{1 + 
             \norm{\m{K}^{-1}_{n, n}}_{L(\ell^\infty; \ell^1)}
             \norm{k}_{\c{C}(\c{X}^2)}}^2
    },
\end{split}
\]
which immediately implies that
\[
\norm{M}_{L(\c{C}(\c{X}^2);\c{C}(\c{X}^2))} 
=
    \sup_{c \neq 0} \frac{\norm{M c}_{\c{C}(\c{X}^2)}}{\norm{c}_{\c{C}(\c{X}^2)}}
\leq 
    n \sbr{1 + \norm{\m{K}^{-1}_{n,n}}_{L(\ell^\infty; \ell^1)} \norm{k}_{\c{C}(\c{X}^2)}}^2
.
\]
Note that, since this bound is independent of the particular realization of the $\ell$-dimensional random Fourier basis $\v{\phi}$ used to construct the approximate prior $\tilde{f}$, it is constant with respect to the expectation \eqref{eqn:covariance_error}. Finally, \textcite{sutherland15} have shown that there exists a constant $C_3$ such that
\[
\E_{\v{\phi}} \norm[1]{k^{(\tilde{f})} - k}_{\c{C}(\c{X}^2)} 
\leq 
\ell^{\nicefrac{-1}{2}} C_{3}
.
\]
Combining this inequality with the preceding ones gives the result.
\end{proof}

Together, Propositions \ref{prop:wasserstein_bound} and \ref{prop:covariance_error} show that error in the approximate prior $\tilde{f}$ controls the error in the resulting approximate posterior $\tilde{f} \given \v{y}$. 
These bounds are not tight, seeing as constants $C_{1}$ and $C_{2}$ both depend on and may grow with the number of observations $n$. 
Based on this observation, it is tempting to think that the error in $\tilde{f} \given \v{y}$ therefore increases in $n$. 
Empirically, however, the opposite trend is observed: the error in $\tilde{f} \given \v{y}$ actually diminishes as $n$ grows \cite{wilson20}.
To better understand this behavior, we now study the conditions under which a pathwise update may counteract the error introduced by an approximate prior.

\subsection{Contraction of approximate posteriors with noise-free observations}
\label{sec:error_contraction}
This section formalizes the following syllogism: (i) the true posterior $f \given \v{y}$ and the approximate posterior $\tilde{f} \given \v{y}$ have the same mean; (ii) as $n$ increases, both posteriors contract to their respective means; (iii) therefore, as $n$ increases, the error introduced by the approximate prior $\tilde{f}$ washes out.

To begin, let $\v{\phi} : \c{M} \to \mathbb{R}^{\ell}$ be an $\ell$-dimensional feature map on an ambient space $\c{M}$ consisting of linearly independent basis functions $\phi_{i}$. 
We will say that $\tilde{f}$ is a \emph{standard normal Bayesian linear model} if it admits the representation
\[
\tilde{f}(\.) &= \sum_{i=1}^\ell w_i \phi_i(\.) 
&
w_i &\~\c{N}(0,1).
\]
This description includes the Karhunen--Lo\`{e}ve and Fourier feature  approximations described in Section~\ref{sec:priors}. As before, let $\m{\Phi}_{n} = \v{\phi}(\m{X}_{n})$ be an $n \times \ell$ feature matrix and $\c{H}_k$ be the reproducing kernel Hilbert space associated with a kernel $k$. We say that a function $\phi_{i}$ \emph{lies locally} in $\c{H}_k$ for a compact $\c{X} \subseteq \c{M}$ if there exists a function $\psi_{j} \in \c{H}_k$ that agrees with $\phi_{i}$ on $\c{X}$, i.e. $\restr{\phi_{i}}{\c{X}} = \restr{\psi_j}{\c{X}}$.

When $\c{M}$ is a compact metric space, the eigenfunctions $\phi_{i}$ used to construct (truncated) Karhunen--Lo\`eve expansions belong to $\c{H}_{k}$ by construction.
More generally, assessing whether or not $\phi_{i}$ lies locally in $\c{H}_k$ is often straightforward for kernels with known reproducing kernel Hilbert spaces.
As a concrete example, the RKHS of a Mat\'{e}rn-$\nu$ kernel is the Sobolev space of order $\kappa = \nu + \nicefrac{d}{2}$. For integer values of $\kappa$, this is the space of square-integrable functions with $\kappa$ square-integrable weak derivatives. 
Trigonometric basis functions $\phi_{i}(\v{x}) = \cos(2 \pi \v{\omega}_{i}^\top \v{x} + \tau_{i})$ can readily be adapted to satisfy this requirement. Specifically, we may multiply them by a suitably chosen, infinitely-differentiable function that ensures they decay to zero outside of $\c{X}$, such that the resulting basis functions (and their derivatives) are square-integrable.

We are now ready to state and prove the primary claim. In the following, Proposition~\ref{prop:kernel_contraction_var} and Corollary~\ref{cor:kernel_contraction_cov} will demonstrate that $\tilde{f} \given \v{y}$ contracts at the same rate as $f \given \v{y}$. Subsequently, Corollary~\ref{cor:posterior_contraction_wasserstein2} will show that the error in $\tilde{f} \given \v{y}$ vanishes as $n \to \infty$ in any reasonable limit where the variance of the true posterior contracts to zero everywhere on $\c{X}$.

\begin{proposition}
\label{prop:kernel_contraction_var}
Suppose $\c{X} \subseteq \c{M}$ is compact and that each of the $\ell$ basis functions $\phi_{i}$ used to construct the standard normal Bayesian linear model $\tilde{f}$ lies locally in $\c{H}_k$. If the points $\m{X}_{n} \subset \c{X}$ used to condition the approximate posterior $\tilde{f} \given \v{y}$ are chosen such that $f\given\v{y}$ satisfies $\sup_{\v{x} \in \c{X}} k^{(f \given \v{y})}(\v{x},\v{x}) \leq \eps$, then it follows that\footnote{This result holds even when the weights are not assumed i.i.d., albeit with a slightly different constant.}
\[
\sup_{\v{x} \in \c{X}} \abs[1]{k^{(\tilde{f} \given \v{y})}(\v{x},\v{x})} \leq C_{4} \eps,
\]
where we have defined $
C_{4} 
= \ell \max_{i}
       \inf \left\{\norm[0]{\psi_i}_{\c{H}_k}^2: \psi_i|_{\c{X}} = \phi_i|_{\c{X}}, \forall \psi_i \in \c{H}_k \right\}
$.
\end{proposition}

\begin{proof} 
Recall from \eqref{eqn:gp_posterior_pathwise_decoupled} we can use the weight function $\v{\xi}(\.) = k(\m{X}_{n},\m{X}_{n})^{-1}k(\m{X}_{n}, \.)$ to express the approximate posterior as $(\tilde{f} \given \v{y})(\.) \,\smash{\eqd}\, \v{\phi}(\.)^\top \v{w} - \v{\xi}(\.)^\top(\v{y} - \m{\Phi}_{n} \v{w})$. 
Under this notation, it is clear that we may immediately upper bound the variance of the $\tilde{f} \given \v{y}$ as
\[
\begin{split}
\Var\del{(\tilde{f} \given \v{y})(\.)}
=
    \E\sbr{\del{\v{\phi}(\.)^\top - \v{\xi}(\.)^\top \m{\Phi}_{n}}\v{w}}^2
\leq
    \ell \max_{i} \del{\phi_i(\.) - \v{\xi}(\.)^\top \phi_i(\m{X}_{n})}^2,
\end{split}
\]
where, on the right, we have used the fact that $\E\norm{\v{w}}^2 = \ell$. By further denoting $\c{G} = \{g \in \c{H}_k: \norm{g}_{\c{H}_k} = 1\}$, we may now exploit the dual representation of the RKHS norm to write
\[
\begin{split}
\label{eqn:kernel_interp_estimate}
\abs{\phi_i(\v{x}_*) - \v{\xi}(\v{x}_{*})^\top \phi_i(\m{X}_{n})}
&\leq
    \norm{\phi_i}_{\c{H}_k} \sup_{g \in \c{G}}\abs{g(\v{x}_*) - \v{\xi}(\v{x}_{*})^\top g(\m{X}_{n})}
    \\
&=
    \norm{\phi_i}_{\c{H}_k} \norm{k(\.,\v{x}_*) - 
    \v{\xi}(\v{x}_{*})^\top \m{K}_{n,*}
    }_{\c{H}_k}
    \\
&=
    \norm{\phi_i}_{\c{H}_k} \underbracket[0.5pt]{\sqrt{k(\v{x}_*,\v{x}_*) - \m{K}_{*,n} \m{K}_{n,n}^{-1}\m{K}_{n,*}}}_{\c{P}_{\m{X}}(\v{x}_*)},
\end{split}
\]
where, because $\phi_i$ lies locally in $\c{H}_k$, we may replace it with any $\psi_i \in  \c{H}_k:\  \psi_i|_{\c{X}} = \phi_i|_{\c{X}}$.
Noting that $\c{P}_{\m{X}}(\.) = \sqrt{k^{(f \given \v{y})}(\.,\.)}$ and collecting terms gives the result.
\end{proof}

\begin{corollary}
\label{cor:kernel_contraction_cov}
With the same assumptions, as $\sup_{\v{x}\in\c{X}} k^{(f \given \v{y})}(\v{x},\v{x}) \-> 0$, it follows that 
\[
\sup_{\v{x}, \v{x}' \in \c{X}} \abs[1]{k^{(\tilde{f} \given \v{y})}(\v{x},\v{x}') - k^{(f \given \v{y})}(\v{x},\v{x}')} \-> 0.
\]
\end{corollary}
\begin{proof}
Begin by applying the triangle inequality to the above and, subsequently, use the Cauchy-Schwartz inequality to bound $k(\v{x},\v{x}') \leq \sqrt{k(\v{x},\v{x})}\sqrt{k(\v{x}',\v{x}')}$, which gives
\[
\begin{split}
\sup_{\v{x}, \v{x}' \in \c{X}} \abs[1]{k^{(\tilde{f} \given \v{y})}(\v{x},\v{x}') - k^{(f \given \v{y})}(\v{x},\v{x}')}
&\leq
\sup_{\v{x}, \v{x}' \in \c{X}} \abs[1]{k^{(\tilde{f} \given \v{y})}(\v{x},\v{x}')}
+
\sup_{\v{x}, \v{x}' \in \c{X}} \abs[1]{k^{(f \given \v{y})}(\v{x},\v{x}')}
\\
&\leq
\sup_{\v{x} \in \c{X}} \abs[1]{k^{(\tilde{f} \given \v{y})}(\v{x},\v{x})}
+
\sup_{\v{x} \in \c{X}} \abs[1]{k^{(f \given \v{y})}(\v{x},\v{x})}.
\end{split}
\]
In the final expression, convergence of the former term is given by Proposition \ref{prop:kernel_contraction_var}, while the latter goes to zero by assumption.
\end{proof}

\begin{corollary}
\label{cor:posterior_contraction_wasserstein2}
With the same assumptions, as $\sup_{\v{x}\in\c{X}} k^{(f \given \v{y})}(\v{x},\v{x}) \-> 0$, it follows that
\[
W_{2,L^2(\c{X})}(f\given\v{y}, \tilde{f} \given \v{y}) \-> 0
.
\]
\end{corollary}

\begin{proof}
Since $L^2(\c{X})$ is a normed space and $\E(f \given \v{y}) = \E(\tilde{f} \given \v{y})$, we have that
\[
W_{2,L^2(\c{X})}(f\given\v{y}, \tilde{f} \given \v{y}) 
= 
W_{2,L^2(\c{X})}\del[3]{
    \underbracket[0.5pt]{f(\.) - \v{\xi}(\.)^\top f(\m{X}_{n})}_{\strut \del[0]{f \given \v{y}}_0}
,\,
    \underbracket[0.5pt]{\v{\phi}(\.)^\top \v{w} - \v{\xi}(\.)^\top \m{\Phi}_{n}\v{w}}_{\strut \del[0]{\tilde{f} \given \v{y}}_0}
},
\]
where $\del[0]{f \given \v{y}}_0$ and $\del[0]{\tilde{f} \given \v{y}}_0$ denote centered processes. Now, let $\mathbbold{0}$ be an almost surely zero stochastic process over $\c{X}$. Then, by the triangle inequality,
\[
W_{2,L^2(\c{X})}\del[2]{\del[0]{f \given \v{y}}_0, \del[0]{\tilde{f} \given \v{y}}_0}
\leq
    W_{2,L^2(\c{X})}\del[2]{\del[0]{f \given \v{y}}_0, \mathbbold{0}}
    +
    W_{2,L^2(\c{X})}\del[2]{\del[0]{\tilde{f} \given \v{y}}_0, \mathbbold{0}}.
\]
Expanding the definition of Wasserstein distances $W_{2,L^2(\c{X})}$ before using Tonelli's theorem to change the order of integration gives
\[
W_{2,L^2(\c{X})}\del[2]{\del[0]{f \given \v{y}}_0, \del[0]{\tilde{f} \given \v{y}}_0}
&\leq
\del[2]{\E\norm[1]{
    \del[0]{f \given \v{y}}_0 - \mathbbold{0}}_{L^2(\c{X})}^2}^{\nicefrac{1}{2}}
+
\del[2]{\E\norm[1]{
    \del[0]{\tilde{f} \given \v{y}}_0 - \mathbbold{0}}_{L^2(\c{X})}^2}^{\nicefrac{1}{2}}
    \nonumber
\\
&=
\del[2]{\int_{\c{X}} k^{(f \given \v{y})}(\v{x},\v{x}) \d \v{x}}^{\nicefrac{1}{2}}
+
\del[2]{\int_{\c{X}} k^{(\tilde{f} \given \v{y})}(\v{x},\v{x}) \d \v{x}}^{\nicefrac{1}{2}},
\]
where both terms in the final expression converge to zero by compactness of $\c{X}$ together with Proposition \ref{prop:kernel_contraction_var}.
\end{proof}
Together, these claims demonstrate that the decoupled approximate posterior $\tilde{f} \given \v{y}$, formed by using the canonical basis $k(\., \m{X}_{n})$ to update a well-specified approximate prior $\tilde{f}$, \emph{inherits} the contractive properties of the true posterior $f \given \v{y}$. 

Per the beginning of this section, approximate priors $\tilde{f}$ defined as standard normal Bayesian linear models with basis functions that lie locally in $\c{H}_k$ are well-specified. The following counterexample helps clarify what can happen when $\tilde{f}$ is misspecified. Consider an approximate prior $\tilde{f} \sim \c{GP}(0, \delta)$ equipped with the Kronecker delta kernel $\delta$ such that 
$\Cov\del[1]{\tilde{f}(\v{x}_i), \tilde{f}(\v{x}_j)}= 1$
if $\v{x}_{i} = \v{x}_{j}$ and $0$ otherwise. Given a finite set of test locations $\m{X}_{*} \subset \c{X} \setminus \m{X}_{n}$, let $\m{\Xi} = \v{\xi}(\m{X}_{*})^\top$. Applying the pathwise update \eqref{eqn:pathwise_update_canonical} to $\tilde{f}$, the posterior covariance is then
\[
\label{eqn:posterior_contraction_counterexample}
\Cov\del[1]{\v{\tilde{f}}_{*} \given \v{y}}
= 
    \Cov\del[1]{\v{\tilde{f}}_{*}}
    + \m{\Xi}\Cov\del[1]{\v{\tilde{f}}_{n}}\m{\Xi}^\top
    - 2 \Cov\del[1]{\v{\tilde{f}}_{*}, \v{\tilde{f}}_{n}}\m{\Xi}^\top
=
    \m{I}
    + \m{K}_{*,n}^{\vphantom{-2}}\m{K}_{n,n}^{-2}\m{K}_{n, *}^{\vphantom{-2}}.
\]
Since the second of the two terms on the right is guaranteed non-negative, the variance of the resulting posterior is bounded from below by $1$.
For this choice of $\tilde{f}$, then, the approximation error inherent to $\tilde{f} \given \v{y}$ does not diminish as $n$ increases.\footnote{Contraction of the true posterior is well-studied and has strong ties to the literature on kernel methods. \textcite{kanagawa2018gaussian} reviews these connections in greater detail: there, Theorem 5.4 shows how the \emph{power function} $\c{P}_{\m{X}}$ can be bounded in terms of the fill distance $h(\m{X}_{n}) = \sup_{\v{x}_{*} \in \c{X}} \infLowerSubscript_{\v{x} \in \m{X}_{n}} \norm{\v{x}_{*} - \v{x}}$.}

\subsection{Sparse approximation errors}
\label{sec:error_sparse_noisy}
We now examine the error introduced by using a sparse pathwise update \eqref{eqn:pathwise_update_sparse} to construct an approximate posterior. As notation, we write $f \given \v{u}$ and $\tilde{f} \given \v{u}$ for the approximate posteriors formed by applying the sparse update to the true prior $f$ and to the approximate prior $\tilde{f}$, respectively. Results discussed here mirror those presented by \textcite{wilson20}. Appealing to the triangle inequality, we have 
\[
\begin{split}
\label{eqn:error_sparse_approximation}
W_{2,L^2(\c{X})} \del[1]{\tilde{f} \given \v{u}, f \given \v{y}} 
&\leq
    \underbracket[0.5pt]{
        W_{2,L^2(\c{X})}\del[1]{\tilde{f} \given \v{u}, f \given \v{u}}
    }_{\hspace*{13ex}\mathclap{\t{error in approximate prior}}\hspace*{13ex}} 
    + 
    \underbracket[0.5pt]{
        W_{2,L^2(\c{X})}\del[1]{f \given \v{u}, f \given \v{y}}
    }_{\hspace*{12.5ex}\mathclap{\t{error in sparse update}}\hspace*{12.5ex}}
\\[-0.625ex]
\mathbb{E}_{\v{\phi}}
\norm[1]{k^{(\tilde{f} \given \v{u})} - k^{(f \given \v{y})}}_{\c{C}(\c{X}^2)} 
&\leq
    \overbracket[0.5pt]{
        \mathbb{E}_{\v{\phi}}\norm[1]{k^{(\tilde{f} \given \v{u})} - k^{(f \given \v{u})}}_{\c{C}(\c{X}^2)}
    }_{\smash{\hspace*{26ex}}}
    + 
    \overbracket[0.5pt]{
        \norm[1]{k^{(f \given \v{u})} - k^{(f \given \v{y})}}_{\c{C}(\c{X}^2)} \vphantom{\mathbb{E}_{\v{\phi}}\norm[1]{k^{(\tilde{f} \given \v{u})} - k^{(f \given \v{u})}}_{\c{C}(\c{X}^2)}}
    }_{\smash{\hspace*{25ex}}}
.
\end{split}
\]
From here, any of the previously presented propositions enable us to control the total error. For the first terms on the right, the same arguments as before lead to the same results; however, the constants involved will change, since the sparse update now assumes the role of the canonical one. The latter terms do not involve the approximate prior and are therefore beyond the scope of our present analysis. Note that similar statements hold for the Gaussian pathwise update \eqref{eqn:pathwise_update_gaussian}.

As a final remark, note that we may reduce the total error \eqref{eqn:error_sparse_approximation} by incorporating additional basis functions $k(\., \m{X})$ into the sparse update. Conceptually, the act of \emph{augmenting} a sparse update amounts to replacing $\v{u} \sim q(\v{u})$ with $\v{u}' \sim q(\v{u}') = p(\v{f} \given \v{u})q(\v{u})$, where $\v{f}$ are process values at centers $\m{X}$ \cite{rasmussen05, quinonero2007approximation}. By construction, $q(\v{u})$ and $q(\v{u}')$ induce the same posterior on $f$. However, because the augmented update utilizes additional basis functions, the error in the induced distribution of $\v{\tilde{f}}_{*}$ diminishes. This result follows from the same line of reasoning as before: since $\E\del[1]{\v{f}_{*} \given \v{u}'} = \E\del[1]{\v{\tilde{f}}_{*} \given \v{u}'}$, $f \given \v{u}'$ and $\tilde{f} \given \v{u}'$ contract to the same function as $\vert \v{u}' \vert \to \infty$. Hence, the approximate prior washes out and the total error decreases.

\section{Applications}
\label{sec:applications}

This section examines the practical consequences of pathwise conditioning in terms of a curated set of representative tasks. Throughout, we focus on how pathwise methods for efficiently generating function draws from GP posteriors enable us to overcome common obstacles and open doors for new research. We provide a general framework for pathwise conditioning of Gaussian processes based on GPflow \cite{matthews2017gpflow}.\footnote{Code is available online at \url{https://github.com/j-wilson/GPflowSampling}.}

\subsection{Optimizing black-box functions}
\label{sec:applications_globalOptimization}
\begin{figure*}
\centering
\includegraphics[width=\textwidth]{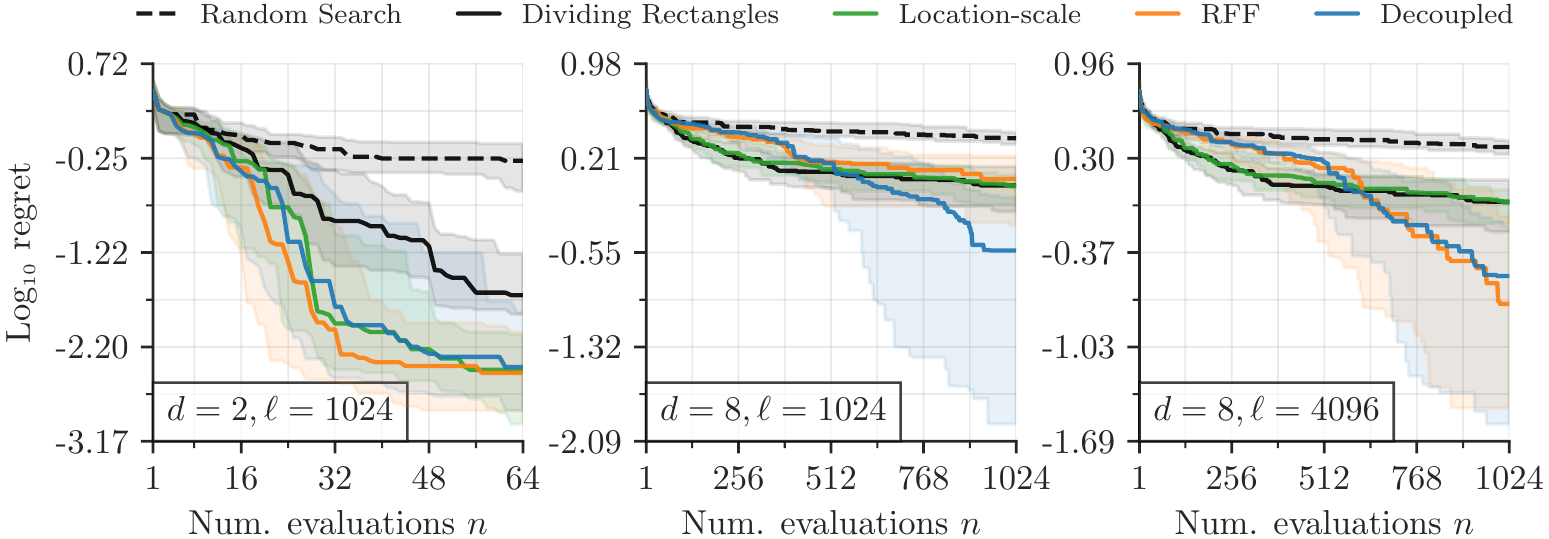}
\caption{Median performances and interquartile ranges of Thompson sampling methods and popular baselines when optimizing function draws from known GP priors on $d = \dim(\c{X})$ dimensional domains. Location-scale Thompson sampling performs well in low-dimensional settings (left), but struggles as $d$ increase due to its inability to efficiently utilize gradient information. RFF posteriors enable us to generate function draws, but demand many more basis functions $b = \ell + n$ than data points $n$ (middle vs. right). Decoupled approaches using canonical basis functions $k(\., \v{x})$ to update RFF priors $\tilde{f}$ avoids these pitfalls and consistently match or outperform competing strategies.}
\label{fig:thompson_sampling_results}
\end{figure*}

Global optimization revolves around the challenge of efficiently identifying a global minimizer
\[
\label{eqn:global_minimizer}
    \v{x}_{\t{min}} &\in \m{X}_{\t{min}} &\m{X}_{\t{min}} = \argmin_{\v{x} \in \c{X}} f(\v{x})
\]
of a black-box function $f: \c{X} \to \mathbb{R}$. Since $f$ is a black box, our understanding of its behavior is limited to a set of observations $\v{y}$ at locations $\m{X}_{n}$. Gaussian processes are a natural and widely used way of representing possible functions $f \given \v{y}$ \cite{movckus1975bayesian,srinivas2009gaussian,frazier2018tutorial}. In these cases, we reason about global minimizers \eqref{eqn:global_minimizer} in terms of a belief over the random set
\[
\label{eqn:global_minimizer_gp}
\m{X}_{\t{min}}^{(f \given \v{y})} = \argmin_{\v{x} \in \c{X}}\,(f \given \v{y})(\v{x}).
\]
Approaches to these problems are often characterized as striking a balance between two competing agendas: the need to learn about the function's global behavior by \emph{exploring} the domain $\c{X}$ and the need to obtain (potentially local) minimizers by \emph{exploiting} what is already known. 

Thompson sampling is a classic decision-making strategy that balances the tradeoff between exploration and exploitation by sampling actions $\v{x} \in \c{X}$ in proportion to the probability that $\v{x} \in \m{X}_{\t{min}}^{(f \given \v{y})}$ \cite{thompson33}. At first glance, this task may seem daunting, since $\m{X}_{\t{min}}^{(f \given \v{y})}$ is random. For a given draw of $f \given \v{y}$, however, $\m{X}_{\t{min}}^{(f \given \v{y})}$ is deterministic. Accordingly, we may Thompson sample an action by generating a function $f \given \v{y}$ and, subsequently, finding a pathwise global minimizer.

Thompson sampling's relative simplicity makes it a natural test bed for evaluating different sampling strategies, while its real-world performance \cite{chapelle11} assures its ongoing relevance in applied settings. A key strength of these methods is that they support embarassingly-parallel batch selection \cite{hernandez2017parallel,kandasamy2018parallelised}. While many GP-based search strategies allow us to choose $\kappa > 1$ queries at a time \cite{snoek2012practical, wilson2018maximizing}, their compute costs tend to scale aggressively in $\kappa$. Especially when evaluations can be carried out in parallel, then, Thompson sampling provides an affordable alternative to comparable approaches.

We considered three different variants of Thompson sampling, corresponding with different approaches to sampling from GP posteriors. The first approach samples random vectors $\v{f}_{*} \given \v{y}$ using location-scale transforms \eqref{eqn:location_scale_gp}; the second approximates posteriors with Bayesian linear models; and, the third updates function draws from $\ell$-dimensional approximate priors $\tilde{f} =\v{\phi}(\.)^\top\v{w}$ using canonical basis functions centered at the $n$ training locations.\footnote{Equation~\eqref{eqn:approximate_posteriors_comparison} highlights the difference between the second and third approaches.} For fair comparison, we allocate $b = \ell + n$ random Fourier basis functions to Bayesian linear models employed by the second approach.

At each round of Thompson sampling, we began by sampling process values $f_i \given \v{y}$ independently on a randomly generated discretization of $\c{X} = [0, 1]^{d}$. Next, we constructed a candidate set $\m{X}_{*}$ using the locations that produces the smallest realizations of $f_{i} \given \v{y}$. Under a location-scale approach, we then jointly sampled process values at $\vert \m{X}_{*} \vert = 2048$ candidates. For both of the alternatives, we instead used $\vert \m{X}_{*} \vert = 32$ candidates to initialize multi-start gradient descent. In all three cases, queries were chosen as minimizers of the resulting vector $\v{f}_{*} \given \v{y}$. Batches of queries were obtained using $\kappa$ independent runs of this algorithm.

To eliminate confounding variables, we experimented with black-box functions drawn from a known Mat\'{e}rn-$\nicefrac{5}{2}$ prior with an isotropic length scale $l = \sqrt{\nicefrac{d}{100}}$ and Gaussian observations $y \sim \c{N}\del{f(\v{x}), 10^{-3}}$. We set $\kappa = d$, but this choice was not found to significantly influence our results. Below, we focus on comparing each Thompson sampling variant's behavior for different amounts of design variables $d$ and basis functions $\ell$.

Figure \ref{fig:thompson_sampling_results} reports key findings based on 32 independent trials; for extended results, see \textcite{wilson20}. First, location-scale methods' inability to use gradient information to efficiently find pathwise minimizers causes its performance to wane as $d$ increases. In contrast, both of the alternative variants of Thompson sampling rely on pathwise-differentiable function draws and, therefore, scale more gracefully in $d$. Second, RFF-based Bayesian linear models struggle to represent posteriors due to variance starvation (Section~\ref{sec:priors_discussion}). As the number of observations $n$ increases relative to the number of basis functions $b = \ell + n$, the function draws they produce come to inadequately characterize the true posterior, causing Thompson sampling to falter. Decoupled approaches to updating $\tilde{f}$ avoid this issue by, e.g., associating the data with the $n$-dimensional canonical basis $k(\., \m{X}_{n})$.

\subsection{Generating boundary-constrained sample paths}
\label{sec:applications_boundary}

\begin{figure*}
\center
\includegraphics[width=\textwidth]{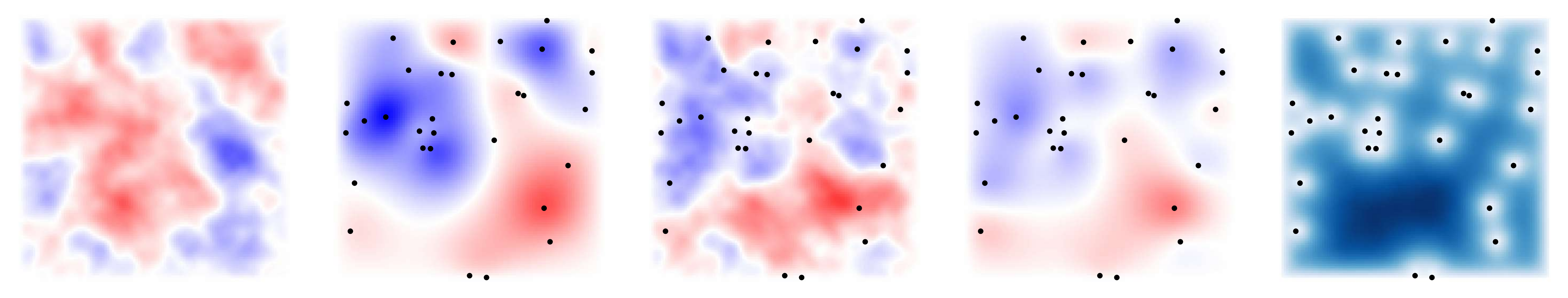}
\includegraphics[width=\textwidth]{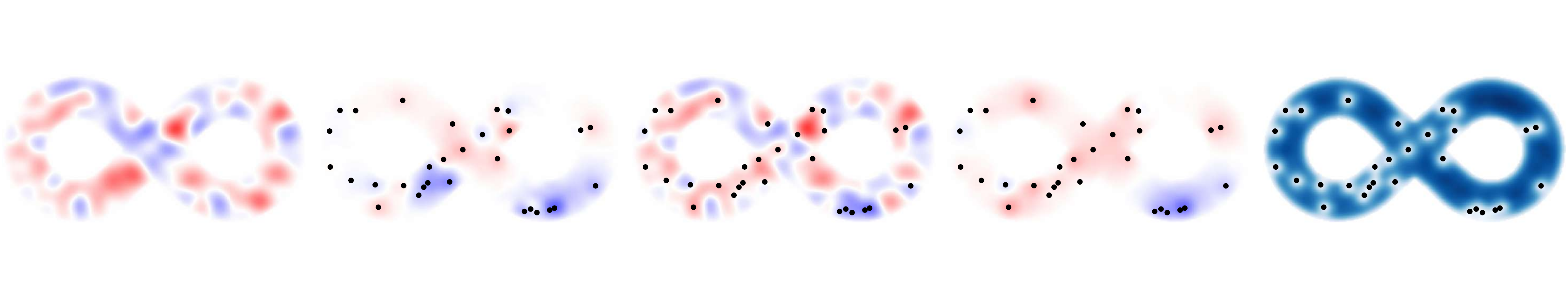}
\caption{Pathwise conditioning of samples from Mat\'{e}rn priors subject to observations $\v{y}$ (black dots) and Dirichlet boundary conditions $f\big|_{\partial \c{X}} = 0$. From left to right, the first three columns show a draw from the prior, a pathwise update, and the corresponding realization of the posterior. The final two columns communicate the empirical mean and standard deviation of the posterior, respectively. \emph{Top:} Illustration of a rectangular domain for which Laplacian eigenpairs are calculated analytically. \emph{Bottom:} A non-trivial domain for which the eigenpairs are approximated numerically.
}
\label{fig:boundary_constrained_samples}
\end{figure*}

This section illustrates how techniques introduced in the preceding sections can be used to efficiently sample Gaussian process posteriors subject to boundary conditions \cite{solin2019know}. \textcite{whittle63} showed that a Mat\'{e}rn GP $f$ defined over $\R^d$ satisfies the stochastic partial differential equation
\[ \label{eqn:matern_spde}
\del{\frac{2 \nu}{\kappa^2} - \Delta}
    ^{\frac{\nu}{2} + \frac{d}{4}}
    f
= \c{W},
\]
where $\c{W}$ is a (rescaled) white noise process, and $\Delta$ is the Laplacian.
Following \textcite{solin2019know} and \textcite{rue2005gaussian}, we restrict \eqref{eqn:matern_spde} onto a (well-behaved) compact domain $\c{X} \subset \R^d$ and impose Dirichlet boundary conditions $\restr{f}{\partial \c{X}} = 0$ to define a boundary-constrained Mat\'{e}rn Gaussian process over $\c{X}$. \textcite{solin2019know} demonstrate that such a prior admits the Karhunen--Lo\`{e}ve expansion
\[
f(\.)
&= \sum_{i=1}^\infty w_i \phi_i(\.)
&
& w_i \~\c{N}\del[3]{0, \frac{\sigma^2}{C_{\nu}} \del{\frac{2 \nu}{\kappa^2} + \lambda_i}^{-\nu-\frac{d}{2}}
},
\]
where $\phi_i$ are eigenfunctions of the \emph{boundary-constrained} Laplacian.
We truncate this expansion to obtain the $\ell$-dimensional Bayesian linear model $\tilde{f}$, which we use together with a pathwise update to construct the posterior.

Figure \ref{fig:boundary_constrained_samples} visualizes function draws from boundary-constrained priors and posterior for two choices of boundaries on $\R^2$, a rectangle and the symbol for infinity. Note that eigenfunctions for rectangular regions of Euclidean domains are available analytically, while those of the infinity symbol are obtained numerically by solving a Helmholtz equation. Examining this figure, we see that the sample paths respect the Dirichlet boundary condition $\restr{f}{\partial \c{X}} = 0$. Karhunen--Lo\`{e}ve expansions enable boundary-constrained GPs, an important class of non-stationary priors, to be used within the pathwise conditioning framework.

\subsection{Simulating dynamical systems}
\label{sec:applications_dynamicalSystems}
\begin{figure*}
\center
\includegraphics[width=\textwidth]{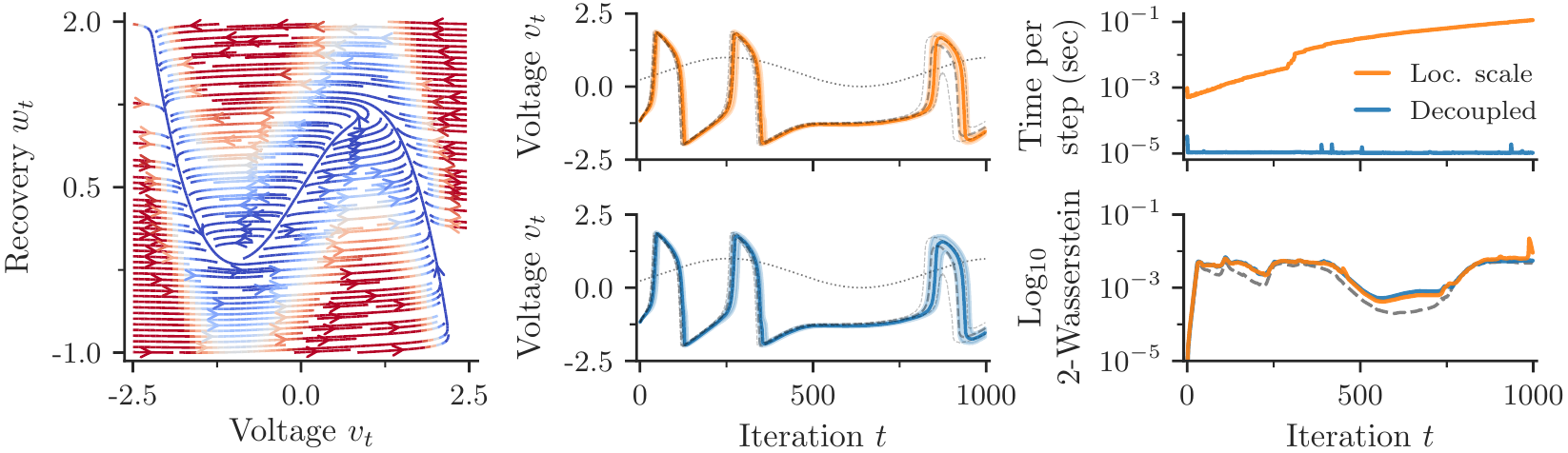}
\caption{Model-based simulations of a stochastic FitzHugh--Nagumo neuron. \emph{Left:} Phase portrait of the true drift function subject to a fixed current $a = 0.5$. 
\emph{Middle:} Empirical medians and interquartile ranges of simulated voltage traces driven by a sinusoidal current (dotted black); ground truth quartiles are shown in dashed gray. 
Trajectories generated via location-scale transforms are summarized on the top in orange, while those produced by decoupled drift functions are portrayed on the bottom in blue.
\emph{Top right:} Comparison of simulation runtimes. \emph{Bottom right:} Sinkhorn
estimates \cite{cuturi2013sinkhorn} to 2-Wasserstein distances between model-based and ground truth state distributions at each step $t$. The noise floor (dashed gray) was found using additional ground truth simulations.}
\label{fig:fitzhugh_nagumo}
\end{figure*}

Gaussian process posteriors are commonly used to simulate complex, real-world phenomena in cases where we are unable to actively collect additional data. These phenomena include dynamical systems that describe how physical states evolve over time. 

We focus on cases where a Gaussian process prior is placed on the \emph{drift} $f: \c{X} \times \c{A} \to \c{X}$ of a time-invariant system, which maps from a state vector $\v{x}_{t} \in \c{X}$ and a control input $\v{a}_{t} \in \c{A}$ to a tangent vector $\v{f}_{t} \in \c{X}$.
Using an Euler--Maruyama scheme to discretize the dynamical system's equations of motion, we obtain the stochastic difference equation (SDE)
\[
\label{eqn:eulerMaruyama}
\v{x}_{t + 1} - \v{x}_{t} 
= \tau f(\v{x}_{t}, \v{a}_{t}) + \sqrt{\tau} \v{\varepsilon}_{t} 
&= \v{y}_{t}
&
\v{\varepsilon}_{t} &\sim \c{N}(\v{0}, \m{\Sigma}_{\v{\varepsilon}}),
\]
where $\tau$ is the chosen step size and $\v{\varepsilon}$ denotes process diffusion.
Together with control inputs $\m{A}_{T} = (\v{a}_{1}, \ldots, \v{a}_{T})$ and diffusion variables $\m{E}_{T} = (\v{\varepsilon}_{1}, \ldots, \v{\varepsilon}_{T})$, each draw of $f$ fully characterizes how an initial state $\v{x}_{1} \sim p(\v{x}_{1})$ evolves over a series of $T$ successive steps.

Since $\v{x}_{t + 1}$ depends on $\v{x}_{t}$, strategies for jointly sampling $\m{X}_{T + 1} = (\v{x}_{1}, \ldots, \v{x}_{T + 1})$ are typically iterative. Under a distributional approach, we generate $\v{x}_{t+1}$ by sampling from the conditional distribution $p(\v{y}_{t} \given \c{D}_{t - 1})$, where $\c{D}_{t - 1}$ denotes the union of the real data $(\v{x}_{i}, \v{y}_{i})_{i=1}^{n}$ and the current trajectory $(\v{x}_{j}, \v{y}_{j})_{j=1}^{t - 1}$. As mentioned in Section~\ref{sec:priors_exact}, we may use low-rank matrix updates to efficiently obtain $p(\v{y}_{t} \given \c{D}_{t - 1})$ from $p(\v{y}_{t} \given \c{D}_{t-2})$ in $\c{O}(t^{2})$ time. Nevertheless, the resulting algorithm suffers from $\c{O}(T^3)$ time complexity. In contrast, approaches based on updating of (approximate) prior function draws scale linearly in $T$.

Many of the same issues were explored by \textcite{ialongo2019overcoming}, who also proposed a linear-time generative strategy for GP-based trajectories. In the language of the present work, this alternative represents the SDE \eqref{eqn:eulerMaruyama} by (i) formulating the unknown drift function as the conditional expectation $\E(f \given \v{u}) = k(\., \m{Z})\m{K}_{m,m}^{-1}\v{u}$ of a sparse Gaussian process $f$ with inducing variables $\v{u} \sim q(\v{u})$ and (ii) defining process diffusion as the sum of the remaining terms $\v{\varepsilon}_{t} \sim \c{N}\del{\v{0}, k^{(f \given \v{u})}(\v{x}_{t}, \v{x}_{t}) + \m{\Sigma}_{\v{\varepsilon}}}$. Similar to the pathwise methods put forth here, this approach avoids inter-state dependencies while unrolling by exploiting the fact each draw of $\v{u}$ realizes an entire drift function.

To better illustrate the practical implications of pathwise approaches to GP-based simulation, we trained a Gaussian process to represent a stochastic variant of the classic FitzHugh--Naguomo model neuron \cite{fitzhugh1961impulses, nagumo1962active}. This model describes a biological neuron in terms of its membrane potential $v_{t}$ and a recovery variable $w_{t}$ that summarizes the state of its ion channels. Written in the form \eqref{eqn:eulerMaruyama}, we have
\[
\label{eqn:fitzhugh_nagumo}
\v{x}_{t + 1} - \v{x}_{t}
=
    \left[\begin{array}{c}
    v_{t + 1} - v_{t} 
    \\ 
    w_{t + 1} - w_{t}
    \end{array}\right]
= 
    \tau \left[\begin{array}{c}
        v_t - \frac{v_t^3}{3} - w_t + a_{t}\\
        \frac{1}{\gamma}(v_t - \beta w_t + \alpha)
    \end{array} \right]
    +
    \sqrt{\tau} \v{\varepsilon}_{t},
\]
where we have chosen $\tau = 0.25\t{ms}$, $\alpha =  0.75$, $\beta = 0.75$, $\gamma = 20$, and $\m{\Sigma}_{\v{\varepsilon}} = 10^{-4}\m{I}$. A two-dimensional phase portrait of this system's drift function given a current injection $a = 0.5$ is shown on the left in Figure~\ref{fig:fitzhugh_nagumo}.

Training data was generated by evaluating \eqref{eqn:fitzhugh_nagumo} for $n = 256$ state-action pairs $(\v{x}_{i}, \v{a}_{i})$, chosen uniformly at random from $\c{X} = [-2.5, 2.5] \times [-1, 2]$ and $\c{A} = [0, 1]$. Changes in each of the state variables were modeled by independent, Mat\'{e}rn-$\nicefrac{5}{2}$ GPs using $m = 32$ inducing variables. Both sparse GPs were trained by minimizing Kullback--Leibler divergences.

At test time, state trajectories were unrolled from steady state for $T = 1000$ steps under the influence of a current injection; see middle column of Figure~\ref{fig:fitzhugh_nagumo}. Drift values $\v{f}_{t}$ were realized using either the $\c{O}(T^{3})$ location-scale technique or the $\c{O}(T)$ pathwise approach. As seen on the right in Figure~\ref{fig:fitzhugh_nagumo}, both strategies are capable of accurately characterizing possible state trajectories. At the same time, their difference in cost is striking: the location-scale method spent 10 hours generating 1000 state trajectories (run in parallel), while the pathwise one spent 20 seconds.

\subsection{Efficiently solving reinforcement learning problems}
\label{sec:applications_rl}

Model-based approaches to autonomously controlling robotic systems often rely on Gaussian processes to infer system dynamics from a limited number of observations \cite{Rasmussen2004, Deisenroth2015, kamthe2017data}. Of these data-efficient methods, we focus on PILCO \cite{Deisenroth2011c}, which is an effective policy search method that uses Gaussian process dynamics models.\footnote{PILCO implementation available separately at \url{https://github.com/j-wilson/GPflowPILCO}.}

Similar to the previous section, we begin by placing a GP prior on the drift function $f: \c{X} \times \c{A} \to \c{X}$ of a black-box dynamical system, now assumed to be deterministic. Rather than being given a sequence of actions $\m{A}_{T}$ and asked to simulate trajectories $\m{X}_{T+1}$, our new goal will be to find parameters $\v{\theta} \in \Theta$ of a deterministic, feedback policy $\pi: \Theta \times \c{X} \to \c{A}$ that maximize the expected cumulative reward
\[
\label{eqn:expected_cumulative_reward}
R(\v{\theta})
= \E_{f,\,\v{x}_{1}}
\sbr{
    \sum_{t=1}^{T}
    r
    \del[2]{\,
        \underbracket[0.5pt]{
            \v{x}_{t} + 
            f\del{\v{x}_{t}, \pi_{\v{\theta}}\del{\v{x}_{t}}}
        }_{\v{x}_{t + 1}}
    \,}}
& = 
    \sum_{t=1}^{T} 
    \E_{\v{x}_{t+1}}\sbr[1]{r(\v{x}_{t+1})}.
\]
For suitably chosen reward functions $r: \c{X} \to \R$, we may optimize $\v{\theta}$ by differentiating \eqref{eqn:expected_cumulative_reward}. The challenge, however, is to evaluate this expectation in the first place. 

The original PILCO algorithm tackles this problem by using moment matching to approximately propagate uncertainty through time. Given a random state $\v{x}_{t} \sim \c{N}(\v{\mu}_{t}, \m{\Sigma}_{t, t})$, we begin by supposing that $\v{x}_{t}$ and $\v{a}_{t} = \pi_{\v{\theta}}(\v{x}_{t})$ are jointly normal. Next, we obtain the corresponding optimal Gaussian approximation to $p(\v{x}_{t}, \v{a}_{t})$ by analytically computing the required moments $\E(\v{a}_{t})$, $\Cov(\v{a}_{t}, \v{a}_{t})$, and $\Cov(\v{a}_{t}, \v{x}_{t})$. This step can also be seen as finding the affine approximation to $\pi_{\v{\theta}}$ that best propagates $\c{N}(\v{\mu}_{t}, \m{\Sigma}_{t, t})$. We now use moment matching to propagate this approximate joint distribution through $f$ in order to construct a second Gaussian approximation, this time to $p(\v{x}_{t}, \v{f}_{t})$.\footnote{By appealing to the affine approximation view of moment matching, we obtain the approximate cross-covariance $\Cov(\v{x}_{t}, \v{f}_{t}) \approx \Cov(\v{x}_{t}, \v{s}_{t})\Cov(\v{s}_{t}, \v{s}_{t})^{-1}\Cov(\v{s}_{t}, \v{f}_{t})$ where $\v{s}_{t} = \v{x}_{t} \oplus \v{a}_{t}$.} By interpreting $\v{x}_{t + 1} = \v{x}_{t} + \v{f}_{t}$ as the sum of jointly Gaussian random variables, we compute the corresponding right-hand side term of \eqref{eqn:expected_cumulative_reward} and, finally, proceed to the next time step.
Overall, this strategy works well when $f$ and $\pi_{\v{\theta}}$ are sufficiently regular and $\c{N}(\v{\mu}_{t}, \m{\Sigma}_{t,t})$ is sufficiently peaked that maps from $\v{x}_{t}$ to $\v{f}_{t}$ are nearly affine in a ball around $\v{\mu}_{t}$ whose radius is dictated by $\m{\Sigma}_{t,t}$.

\begin{figure*}
\center
\includegraphics[width=\textwidth]{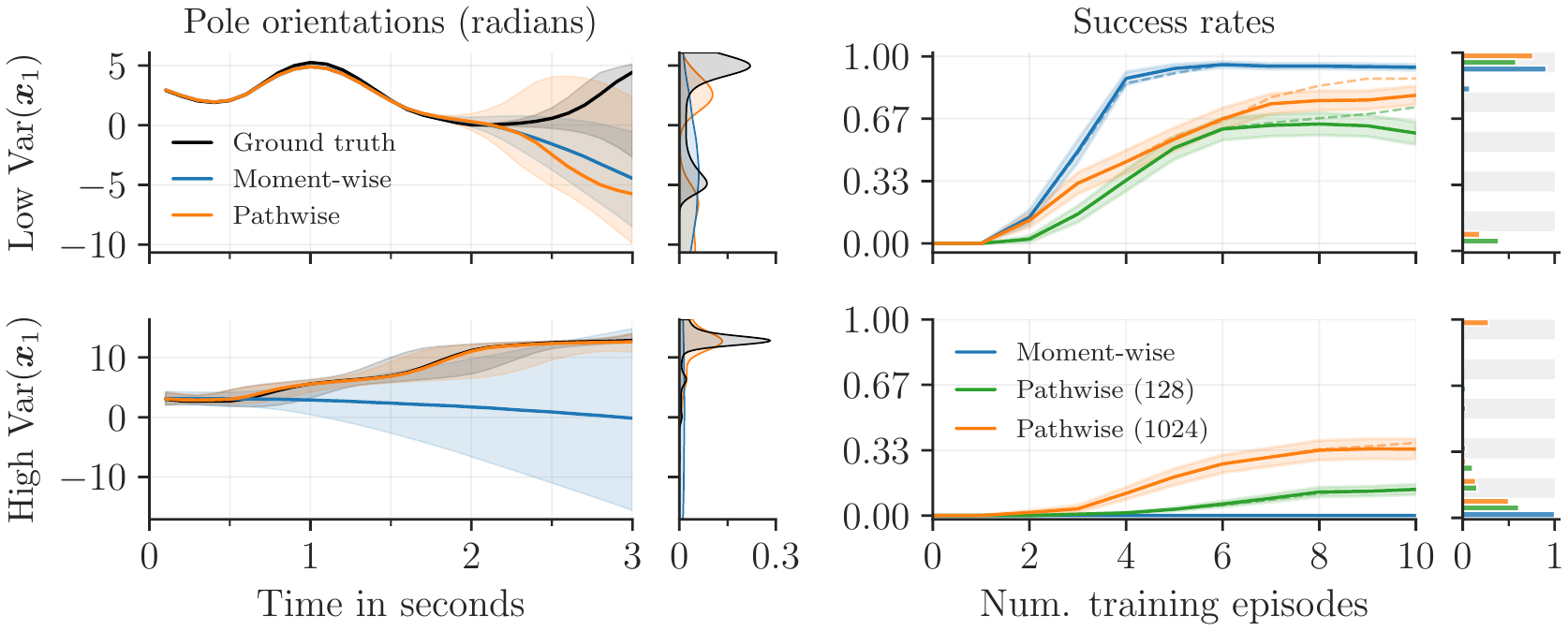}
\caption{Behavior and performance of PILCO algorithms applied to different versions of cart-pole. Marginal distributions of terminal values are shown immediately to the right of each plot. In top and bottom rows, initial state $\v{x}_{1}$ is nearly deterministic and highly randomized, respectively. \emph{Left:} Medians and interquartile ranges of simulated pole orientations. \emph{Right:} Means and standard errors of success rates (estimated separately by unrolling the true system 100 times); dashed lines represent average performances of incumbent policies. On the bottom right, Pathwise $(s)$ indicates that $s$ samples were used during training.
}
\label{fig:cartpole}
\end{figure*}

Here, we are interested in comparing the behavior of moment-based and path-based approaches to optimizing \eqref{eqn:expected_cumulative_reward}. 
To shed light on how these approaches fare in the context of typical learning problems, we experimented with both methods on the \emph{cart-pole} task \cite{barto1983neuronlike}, which consists of moving a cart horizontally along a track in order to swing up and balance a pole, upside down, at a target location.
State vectors $\v{x} = [x_{0}, \dot{x}_{0}, x_{1}, \dot{x}_{1}]^{\top}$ define the position of the cart $x_{0}$, angle of the pole $x_{1}$, and time derivatives thereof; while, actions $\v{a} \in \c{A} = [-10, 10]\,\t{N}$ represent the lateral forces applied to the cart.

We follow \textcite{Deisenroth2015} by using a $0.5 \t{m}$ long, $0.5 \t{kg}$ pole and a $0.5 \t{kg}$ cart with a $0.1 \t{Ns/m}$ friction coefficient. Each episode ran for a length of 3 seconds, discretized at $0.1 \t{s}$ intervals during which time actions were held constant, i.e., zero-order hold control. We set the goal state to $\v{x}_{\t{goal}} = \v{0}$ and define rewards according to a Gaussian function
\[
r(\v{x}) = \exp\bigg(
    \!-\frac{1}{2}
    \underbracket[0.5pt]{(\v{x} - \v{x}_{\t{goal}})^\top
    \m{\Lambda}^{-1}
    (\v{x} - \v{x}_{\t{goal}})}_{\substack{\t{sq. Euclidean distance}\\ \t{between pendulum tip and goal}}}
\bigg),
\]
whose precision matrix $\m{\Lambda}^{-1}$ was chosen such that the bracket term is proportional to the squared Euclidean distance between (the Cartesian coordinates of) the tip of the pole in states $\v{x}$ and $\v{x}_{\t{goal}}$. Along the same lines, an episode was considered successful if the tip of the pole was within $0.1 \t{m}$ of the goal for 10 or more consecutive time steps. Depending on the particular experiment, states were initialized in one of two ways: (i) the standard case $\v{x}_{1} \sim \c{N}\del{[0, \pi, 0, 0]^\top, 0.01 \m{I}}$ or (ii) a challenge variant $\v{x}_{1} \sim \c{N}\del{[0, \pi, 0, 0]^\top, \diag(1, 1, \pi, \pi)}$.

In all cases, system dynamics were represented by a set of independent sparse GPs with squared exponential kernels, each of which predicted a single component of the tangent vector $\v{f} = f(\v{x}, \v{a})$. Upon collecting an additional episode of training data, these GPs were trained from scratch using L-BFGS \cite{liu1989limited} with $m = \min(n, 256)$ inducing variables, whose corresponding inducing locations $\m{Z}$ were initialized via $k$-means.

We defined policies as kernel regressors with inverse link functions $g^{-1} : \R \to [-10, 10]$
\[
\pi_{\v{\theta}}(\.)
&= g^{-1}\del[4]{\sum_{i=1}^{30} w_{i} k(\., \v{x}_{i})}
&
g^{-1}(\.) &= 20 \Phi(\.) - 10,
\]
where $k$ denotes a squared exponential kernel and $\Phi: \R \to [0, 1]$ is the standard normal CDF. Policy parameters $\v\theta$ consisted of centers $(\v{x}_{1}, \ldots, \v{x}_{30})$, weights $\v{w}$, and length scales $\v{l}$. Following \textcite{Deisenroth2015}, policies were initialized once after collecting a random initial episode and subsequently fine-tuned. At each round, $\v{\theta}$ was updated 5000 times using ADAM \cite{kingma2014adam} with gradient norms clipped to one and an initial learning rate $0.01$ that decreased by a factor of ten after every third of training. Pathwise approaches propagated uncertainty by unrolling a separate draw of $\v{x}_{1}$ along each realization of $f$, both of which were resampled prior to each update of $\v{\theta}$.

In line with previous findings, moment-wise PILCO consistently solves the standard cart-pole task within a few episodes \cite{Deisenroth2015}. As initial state distributions become increasingly diffuse, however, moment matching struggles to accurately propagate uncertainty. As seen in the bottom row of Figure~\ref{fig:cartpole}, this inability prevents moment-wise PILCO from learning meaningful policies for the challenge variant of cart-pole. Pathwise alternatives do not experience this issue, but they are not without their own shortcomings. We now discuss relative merits of both approaches to propagating uncertainty. 

Pathwise uncertainty propagation is significantly faster than moment matching, enabling us to simulate (tens of) thousands of trajectories in the time it takes to complete a single forward pass of moment matching. As Monte Carlo methods, pathwise estimates of \eqref{eqn:expected_cumulative_reward} allow us to easily achieve the desired balance of accuracy and cost by controlling the sample size. Here, the use of sampling conveys additional benefits. First, it frees us from the restrictive class of moment matchable models by eliminating the need for closed-form integration. Second, it drastically simplifies implementation and allows us to fully take advantage of modern hardware and software, such as GPUs and automatic differentiation.

On the other hand, we observe that moment-wise uncertainty propagation  sometimes improves performance. By locally linearizing the functions it permeates, moment matching implicitly favors simpler, smoother dynamics $f$ and policies $\pi_{\v{\theta}}$ (see Figure~\ref{fig:cartpole}). Perhaps for this very reason, moment-wise PILCO was found to train more robustly. In particular, its pathwise counterpart was more susceptible to catastrophic forgetting: after solving the problem during the previous round of training, policies trained via pathwise uncertainty propagation were more likely to diverge. To illustrate this behavior, we define the \emph{incumbent} as the policy that achieves highest expected reward under the model $f$. Unlike those of its moment-wise analogue, pathwise PILCO's incumbents (dashed lines) often outperform more recent policies (solid lines) by significant margins; see right side of Figure~\ref{fig:cartpole}. While this issue was easy to reproduce, the relatively abundance of moving pieces makes it difficult to pinpoint precisely why it occurs. 

Many of the challenges highlighted above are common in reinforcement learning, where generic solutions are often outperformed by skillfully tuned, bespoke alternatives. Nevertheless, we hope that the ease and flexibility of pathwise approaches to simulating posteriors will allow Gaussian processes to be applied to a wide range of problems where data-efficiency and uncertainty calibration are paramount.

\subsection{Evaluating deep Gaussian processes}
\label{sec:applications_models}

When applying Gaussian process methods to novel problems, we are often faced with a natural dilemma: many phenomena of interest are definitively non-Gaussian. In order to leverage Gaussian processes to model these phenomena, we typically resort to nonlinearly transforming $f$. Seeing as Gaussian random variables pushed forward through nonlinear functions seldom admit convenient analytic expressions, we are forced to trade tractability for expressivity.

This issue has recently come to the fore in the context of deep Gaussian processes \cite{damianou2013deep}, which represent function priors as compositions
\[
\label{eqn:deep_gp}
f(\.) = 
    \del[1]{f^{(T)} \circ \ldots  \circ f^{(2)} \circ f^{(1)} }
    \del[1]{\.},
\]
where $f^{(t)} \sim \c{GP}\del{\mu^{(t)}, k^{(t)}}$ for $t=1,\ldots,T$.
Following \textcite{Salimbeni2017}, sample-based methods have become the standard approach for evaluating and training these models. 
When a composition \eqref{eqn:deep_gp} consists of independent layers made up of independent, scalar-valued GPs (or linear combinations thereof), $f(\v{x})$ can be efficiently sampled without resorting to expensive matrix operations. When these assumptions are violated, however, sample-based evaluations of deep GPs quickly becomes expensive. One such example was implicitly touched on in preceding sections: Gaussian process models of time-varying stochastic differential equations can be seen as continuous-time analogues of certain deep GPs \cite{hegde2018deep}. In these cases, dependencies between successive evaluations of a GP-based drift function $f^{(t)}(\.) = f(t, \.)$ cause location-scale based evaluations to grind to halt (see Section~\ref{sec:applications_dynamicalSystems}).

Similar issues arise when sampling from compositions of multioutput GPs \cite{van2020framework}. The remainder of this section focuses on the particular case of deep convolutional GPs \cite{blomqvist2019deep, dutordoir2020bayesian}. Here, a deep GP is defined in close analogy to a convolutional neural network \cite{van2017convolutional}: each layer consists of a set of independent maps that are convolved over local subsets (patches) of an image $\rv{x}_{t} \in \mathbb{R}^{ c_{t} \times h_{t} \times w_{t}}$. For a convolutional neural network, these \emph{patch response functions} are affine transformations followed by nonlinearities; while, for a convolutional Gaussian process, they are draws from GP posteriors.

Since each of the $c_{t}$ independent patch response functions produces $h_{t} \times w_{t}$ output features, the covariance of the Gaussian random variables $\rv{x}_{t} = f^{(t - 1)} * \v{x}_{t-1}$ is a block diagonal square matrix of order $c_{t} \times h_{t} \times w_{t}$. Location-scale approaches to jointly sampling these feature maps incur $\c{O}\left(c_{t} \times h_{t}^{3} \times w_{t}^{3}\right)$ cost when computing matrix square roots.\footnote{This cost is separately incurred by each input to each layer, see \textcite{dutordoir2020bayesian}.} Rather than sampling each layer at the current set of inputs, pathwise strategies sample entire models. Said again, pathwise approaches operate by drawing deterministic models from (approximations to) deep GP posteriors.\footnote{Here, we have assumed the use of approximate priors akin to those discussed in Section~\ref{sec:priors}.} By doing so, these methods allow us to evaluate individual layers in $\c{O}\left(c_{t} \times h_{t} \times w_{t}\right)$ time.

\begin{figure*}
\center
\includegraphics[width=\textwidth]{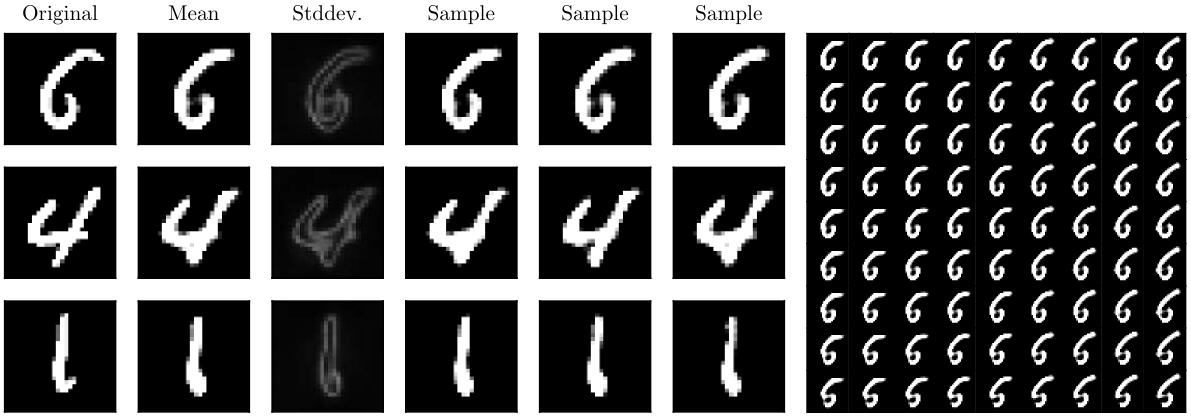}
\caption{Reconstructions of MNIST digits by a deep convolutional GP trained to act as an autoencoder. \emph{Left:} Mean and standard deviations of the (non-Gaussian) distribution over the reconstructions of randomly chosen test images are shown alongside three independently generated samples. \emph{Right:} A 2-dimensional projection of a 25-dimensional latent space is found by performing SVD on the Jacobian of the mean response of the first decoder layer given an encoding of first image shown on the left. Reconstructions using the mean of each decoder layer are shown for a local walk in this 2-dimensional projected space.}
\label{fig:mnist_autoencoder}
\end{figure*}

As an illustrative example, we trained a deep GP to act as an autoencoder for the MNIST dataset \cite{lecun-mnisthandwrittendigit-2010}. For the encoder, we employed a sequence of three convolutional layers, each with 384 inducing patches $\m{Z} \in \mathbb{R}^{c_{t-1} \times 3 \times 3}$ shared between $c_{t} \in (32, 32, 1)$ independent GPs. Strides and padding were chosen to produce a 25-dimensional encoding of a 784-dimensional image. Analogously, we defined the decoder using three transposed convolutional layers, each with 384 inducing patches $\m{Z} \in \mathbb{R}^{c_{t-1} \times 3 \times 3}$ shared between $c_{t} \in (32, 32, 32)$ independent GPs. We then used a final decoder layer, consisting of a single convolutional GP (with the same general outline as above), to resolve penultimate feature maps $\mathbb{R}^{32 \times 28 \times 28}$ into image reconstructions $\mathbb{R}^{1 \times 28 \times 28}$. In all cases, we employ residual connections by using bilinear interpolation to define identity mean functions. Following \textcite{Salimbeni2017}, we initialized inducing patches $\m{Z}$ using $k$-means and inducing distributions to be nearly deterministic. 

Model evaluations were performed by using the sparse update \eqref{eqn:pathwise_update_sparse} together with functions drawn from approximate priors constructed using $\ell = 256$ random Fourier features. We associate each input image with a single draw of the model. Running on a single GPU, the model outlined above was jointly trained in just over 40 minutes using $10^{4}$ steps of gradient descent with a batch size of 128. Figure~\ref{fig:mnist_autoencoder} visualizes the behavior of reconstructions for a randomly chosen set of test images. While this GP-based autoencoder performs fairly well, there is an abundance of open questions regarding deep Gaussian processes in the wild. We hope that the ability to efficiently sample and evaluate draws of composite functions \eqref{eqn:deep_gp} will enable future works to further explore this space.

\section{Conclusion}
\label{sec:conclusion}
Throughout this work, we have used Matheron's update rule (Theorem~\ref{thm:matheron_finite}) as the driving force for looking at Gaussian processes in a different light. This simple equivalence, namely
\[
(\v{a} \given \v{b} = \v{\beta}) 
\eqd 
    \v{a} + \m{\Sigma}_{\v{a},\v{b}}^{\vphantom{-1}} \m{\Sigma}_{\v{b},\v{b}}^{-1}(\v{\beta} - \v{b}),
    \tag{\ref{eqn:matheron_finite}}
\]
allows us to think about GP posteriors at the level of sample paths. Doing so not only helps to clarify existing ideas, but enables us to envision new ones. As it turns out, many of these ideas are intimately practical. 

We have repeatedly stressed how pathwise conditioning enables us to separate Gaussian process priors from data-drive updates. We may then leverage these objects' mathematical properties to construct efficient approximators. As a rule, however, the patterns at play in both cases are fundamentally different: priors typically admit convenient global trends, whereas data often exerts localized influences. Fully exploiting these properties requires us to use different representations, such as different bases, for each of these terms. Decomposing GP posteriors into global and local components makes this particularly easy.

Pathwise and distributional conditioning are complementary viewpoints that lead to complementary methods. In cases where quantities of interest are readily obtained by working with (finite-dimensional) marginals, distributions act as a natural lens for viewing Gaussian process posteriors. 
On the other hand, when a problem involves arbitrarily many random variables, random functions provide a more direct path to efficient solutions.

All said and done, pathwise conditioning is a powerful tool for both reasoning about and working with GPs. Methods that fit this mold are generally straightforward to use and can easy be tailored to take advantage of a given task's properties. We have done our best to overview key ingredients for efficiently sampling from Gaussian process posteriors and look forward to learning more about related ideas alongside you, the reader.

\acks
We are grateful to Prof. Mikhail Lifshits for his helpful comments regarding the theoretical part of this work.
J.T.W. was supported the EPSRC Centre for Doctoral Training in High Performance Embedded and Distributed Systems, reference EP/L016796/1.
V.B. and P.M. were supported by "Native towns", a social investment program of PJSC Gazprom Neft and by the Ministry of Science and Higher Education of the Russian Federation, agreements N\textsuperscript{\underline{o}} 075-15-2019-1619 and N\textsuperscript{\underline{o}} 075-15-2019-1620.
A.T. was supported by the Department of Mathematics at Imperial College London.

\printbibliography

\end{document}